\documentclass{article}

\PassOptionsToPackage{numbers, compress}{natbib}
\usepackage[preprint]{neurips_2026}

\usepackage[utf8]{inputenc} 
\usepackage[T1]{fontenc}    
\usepackage{hyperref}       
\usepackage{url}            
\usepackage{booktabs}       
\usepackage{amsfonts}       
\usepackage{nicefrac}       
\usepackage{microtype}      
\usepackage{xcolor}         
\usepackage{multirow}
\usepackage{enumitem}

\newcommand{\todo}[2][]{}
\usepackage{subcaption}
\usepackage{graphicx}
\usepackage{tcolorbox}

\definecolor{wikiRed}{RGB}{139, 0, 0}      
\definecolor{ubBlue}{HTML}{005BBB}
\definecolor{synOrange}{RGB}{204, 102, 0}  
\definecolor{kaustOrange}{HTML}{F68F1E}
\usepackage{hyperref}

\usepackage{algorithm}
\usepackage{algorithmic}

\usepackage{amsmath}
\usepackage{amssymb}
\usepackage{mathtools}
\usepackage{amsthm}
\usepackage{xcolor}
\usepackage{wrapfig}

\usepackage[capitalize,noabbrev]{cleveref}

\theoremstyle{plain}
\newtheorem{theorem}{Theorem}[section]

\newtheorem{lemma}[theorem]{Lemma}

\theoremstyle{definition}

\definecolor{plotblue}{RGB}{70,105,215}
\definecolor{plotlightblue}{RGB}{239,246,255}

\newtcolorbox{question}[1][]{colback=plotlightblue, colframe=ubBlue, boxrule=0.5pt, arc=2pt, #1}

\title{In-Run Data Shapley for Adam Optimizer}
\hypersetup{
  pdftitle={In-Run Data Shapley for Adam Optimizer},
  pdfauthor={Meng Ding, Zeqing Zhang, Di Wang, Lijie Hu}
}

\author{%
  Meng Ding$^{1,2}$\thanks{Equal contribution.} \quad
  Zeqing Zhang$^{1}$\thanks{Equal contribution.} \quad
  Di Wang$^{3}$ \quad
  Lijie Hu$^{1}$\thanks{Corresponding author: \texttt{lijie.hu@mbzuai.ac.ae}.} \\
  $^{1}$MBZUAI \quad $^{2}$UMass Boston \quad $^{3}$KAUST \\
}

\begin{document}

\maketitle


\begin{abstract}
Reliable data attribution is essential for understanding, debugging, and curating modern machine learning training data, with the Shapley value serving as a principled framework for data valuation. 
Recent In-Run Data Shapley methods avoid the prohibitive cost of retraining by decomposing data value into step-wise contributions along a single realized training trajectory. 
However, existing scalable in-run estimators are derived from the linear update structure of stochastic gradient descent (SGD), while modern deep learning pipelines are widely trained with adaptive optimizers such as Adam and AdamW. 
In this work, we show that data attribution is optimizer-dependent: SGD-induced and Adam-induced data values exhibit extremely low agreement in both magnitude and ranking, indicating that SGD-based attribution can lead to misleading rank-based decisions under Adam training. 
To address this mismatch, we propose \emph{Adam-Aware In-Run Data Shapley}, an optimizer-consistent attribution framework for Adam training. 
We define a fixed-state Adam local utility and derive a first-order approximation that explicitly accounts for momentum and coordinate-wise variance normalization. 
To make the estimator scalable, we introduce a \emph{Linearized Ghost Approximation}, which restores an additive dot-product form without materializing per-sample gradients. 
Experiments show that our method achieves near-perfect fidelity to exact local Shapley values under the fixed-state Adam utility ($R>0.99$), substantially improves efficiency over direct per-sample Adam-aware computation, and improves downstream rank-based attribution tasks including semantic source identification and data pruning.
\end{abstract}

\vspace{-0.1in}\section{Introduction} \vspace{-0.1in}

Modern machine learning models rely heavily on large-scale training data \cite{gao2020pile,raffel2020exploring}, yet their performance and behavior can be substantially shaped by the quality and influence of individual training examples. 
The lack of reliable data attribution mechanisms can lead to degraded performance and computational inefficiency. Without principled mechanisms to isolate and evaluate sample influence, trained models may amplify harmful biases or remain vulnerable to data poisoning attacks \cite{koh2017understanding,ghorbani2019data,DBLP:journals/corr/abs-2507-04059}. Furthermore, the presence of problematic data escalates the computational cost of modern training pipelines by wasting resources on uninformative or harmful examples \cite{paul2021deep,DBLP:journals/corr/abs-2501-15963,hu2025dissectingrepresentationmisalignmentcontrastive}.
Therefore, an actionable and reliable attribution method should quantify how each training example contributes to the specific model produced by the realized training pipeline.

To fairly attribute data contributions, the \textit{Shapley value} has emerged as the gold standard. Originating from cooperative game theory \cite{shapley1953value}, the Shapley value provides a unique, unbiased distribution of the total model performance among individual data points based on their marginal contributions, emerging as a widely adopted framework \cite{ghorbani2019data,jia2019towards}. While it has desirable properties, such as fairness, efficiency, and additivity \citep{Jia2019EfficientTD}, the classic calculation of Shapley values requires retraining models on many subsets of the training data, which is computationally infeasible for modern deep learning pipelines. Recent work on In-Run Data Shapley addresses this scalability barrier by decomposing data value into step-wise contributions along a single realized training trajectory~\citep{wang2025data}. This in-run perspective is especially appealing for modern models because it attributes data contribution to the specific training run that produced the deployed model, rather than to an average over many hypothetical retraining runs.

However, this in-run perspective also raises a natural question: are data values consistent across different optimizers? Existing In-Run Data Shapley is derived under the update rule of stochastic gradient descent (SGD), where each parameter update is a linear combination of per-sample gradients. This linearity reduces the local Shapley value to a gradient-gradient inner product and enables efficient ghost-style computation. In contrast, modern deep learning models are widely trained with adaptive optimizers such as Adam \cite{kingma2014adam} and AdamW \cite{loshchilov2017decoupled}, where the update direction is determined by historical momentum and coordinate-wise variance normalization. 
As a result, an SGD-based In-Run attribution score may not faithfully capture the contribution of a data point under Adam training, since it corresponds to different optimizer-induced training dynamics.

In this paper, we study optimizer-consistent In-Run Data Shapley for Adam training, with the goal of attributing data contribution along the realized Adam optimization trajectory while preserving the scalability of in-run attribution. Our contributions are summarized as follows:

\noindent\textbf{1. Optimizer-Aware Data Attribution.}  
    We demonstrate that data value is not an intrinsic property of the dataset, but is fundamentally coupled to the optimization trajectory. 
    Using retraining-based Truncated Monte Carlo (TMC) Shapley as a model-agnostic reference, we show that SGD-induced and Adam-induced data values exhibit extremely low agreement in both magnitude and ranking (Pearson $R=0.0579$, Spearman $\rho=0.0465$), indicating that SGD-based attribution can lead to misleading rank-based decisions under Adam training.

\noindent\textbf{2. Adam-Aware In-Run Data Shapley.} 
    We derive the first Adam-aware In-Run Data Shapley estimator. 
    By defining a fixed-state local utility under Adam and applying a first-order approximation to the Adam update map, our formulation explicitly accounts for momentum and coordinate-wise variance normalization along the realized Adam trajectory.

\noindent\textbf{3. Scalable ``Linearized Ghost'' Computation.} 
    To overcome the nonlinearity of Adam updates, which prevents standard efficient aggregation, we introduce the Linearized Ghost Approximation. 
    This technique restores an additive structure in the local utility, allowing Adam-aware attribution to be computed through ghost-style dot-products without materializing per-sample gradients. In our GPT-2 Small benchmark, it improves throughput by $2.24\times$ and reduces peak memory by $69.3\%$ compared with a naive direct implementation.

\noindent\textbf{4. Fidelity and Practical Utility.} 
    We show that our method achieves near-perfect fidelity to explicit Adam marginal utility changes ($R>0.99$), substantially outperforming SGD-based proxies under Adam dynamics. 
    We further demonstrate its usefulness in downstream rank-based attribution tasks, including semantic source identification and data pruning.

\vspace{-0.1in}\section{Related Work}\vspace{-0.1in}
\label{sec:related_work}

\noindent\textbf{Data Attribution Methods.}
Data attribution quantifies the influence of individual data points on model predictions. 
A foundational approach uses \emph{influence functions}, originating from statistics~\citep{cook1980characterizations}. 
In deep learning, \citet{koh2017understanding} apply first-order Taylor expansions and inverse Hessian-vector products to estimate the effect of training samples. 
Several works improve the scalability of influence-function computation. 
For example, LiSSA provides a stochastic approximation to inverse Hessian-vector products~\citep{agarwal2017second}, while recent methods such as DataInf and HyperINF further accelerate influence estimation for large-scale models and parameter-efficient fine-tuning settings~\citep{kwon2023datainf,zhou2024hyperinf}. 
Influence-function methods have also been extended to study large language models~\citep{grosse2023studying}, but they can be unstable in non-convex settings~\citep{basu2020influence}. 
Trajectory-based methods provide another line of data attribution. 
For example, TracIn~\citep{pruthi2020estimating} avoids explicit Hessian inversion by accumulating gradient dot-products along the training trajectory. 
A related line of work differentiates through the training process, including approximate unrolling and metagradient-based attribution~\citep{bae2024training,ilyas2025magic}. 
These methods estimate how a data perturbation propagates through future optimization steps and affects the final trained model. 
Other methods take different perspectives on data value. 
Datamodels~\citep{ilyas2022datamodels} learn mappings from training subsets to model outputs at a higher computational cost. 
DAVINZ~\citep{wu2022davinz} estimates data value at initialization, while validation-free methods remove dependence on a held-out validation set~\citep{xu2021validation}. 
In-Run Data Shapley takes a different perspective: it defines a step-wise local utility along the realized training trajectory and aggregates local Shapley values across iterations~\citep{wang2025data}. 
Our work follows this in-run perspective and addresses a different question: how to make the step-wise Shapley utility consistent with the optimizer that actually generated the trajectory. 

\noindent\textbf{Data Shapley Value.}
To quantify data value, prior work adopts the Shapley value from cooperative game theory~\citep{shapley1953value}. In machine learning, Data Shapley~\citep{ghorbani2019data} treats training samples as players and model performance as the payoff, but exact computation is NP-hard, motivating efficient approximations.
The most common approach is Monte Carlo Shapley, which estimates values by averaging marginal contributions over random permutations~\citep{mitchell2022sampling}. Early work combined this estimator with truncation techniques~\citep{jia2019towards, wang2023note}, while later studies introduced more advanced sampling strategies to improve efficiency and reduce variance, including ergodic sampling~\citep{illes2019estimation}, stratified empirical Bernstein sampling~\citep{burgess2021approximating}, and multilinear sampling schemes~\citep{okhrati2021multilinear}. Recent work further improves scalability through randomized experimental designs~\citep{lin2022measuring} and stochastic amortization~\citep{covert2024stochastic}.
Beyond efficiency, several extensions aim to improve robustness and interpretability. Beta Shapley~\citep{kwon2021beta} emphasizes low-cardinality coalitions, while Data Banzhaf~\citep{wang2023data} and Weighted Banzhaf~\citep{li2023robust} provide alternative influence measures based on the Banzhaf power index.


Despite these advances, many existing approaches still rely on retraining surrogate models or computationally intensive sampling procedures. The recently proposed \emph{In-Run Data Shapley}~\citep{wang2025data} avoids retraining by decomposing the Shapley value into per-iteration training updates. However, its theoretical derivation critically depends on the linearity of stochastic gradient descent, limiting its applicability to stateful and non-linear optimizers such as Adam. Addressing this gap is the focus of the present work.

\vspace{-0.1in}\section{Background}\vspace{-0.1in}

In this section, we introduce the formal setup for data attribution and review the Shapley value. 
We then describe the In-Run Data Shapley framework, which decomposes data value along a single realized training trajectory. 
Finally, we revisit the SGD and Adam update rules, highlighting the key structural difference that motivates our Adam-aware formulation.

\noindent\textbf{Problem Setup.}
Consider a training dataset $D=\{z_i\}_{i=1}^N$ consisting of $N$ individual data points. In data attribution, our goal is to assign each training point $z\in D$ a scalar value that measures its contribution to the performance or behavior of a trained model. Let $U(\cdot)$ denote a utility function that maps a subset of training data $S\subseteq D$ to a real-valued score. For example, $U(S)$ can be the validation accuracy, negative validation loss, or another task-specific performance measure of a model trained on $S$. 
Given a utility function $U$, a data attribution method assigns each point $z\in D$ a scalar score, denoted by $\phi_z(U)$, to quantify its contribution to the chosen utility.

\noindent\textbf{Data Shapley Value.}
The Shapley value, originating from cooperative game theory~\citep{shapley1953value}, provides a principled way to distribute the total utility among individual data points. 
In this formulation, each data point is treated as a player, and the model utility is treated as the payoff of the coalition. 
For any data point $z\in D$, its Data Shapley value with respect to utility $U$ is defined as
{\small{\[
\phi_z(U)
:=
\frac{1}{N}
\sum_{k=1}^{N}
\binom{N-1}{k-1}^{-1}
\sum_{\substack{S\subseteq D\setminus\{z\}\\ |S|=k-1}}
\left[
U(S\cup\{z\})-U(S)
\right],
\]}}
where $D_{-z} = D \setminus \{z\}$ represents the dataset with point $z$ excluded and $S$ indicates a subset of $D_{-z}$ with cardinality $|S| = k-1$. The term $[U(S \cup \{z\}) - U(S)]$ measures the marginal contribution of data point $z$ when added to the existing subset $S$.
The popularity of the Shapley value stems from the fact that it is the unique notion of data value satisfying four axioms: Null Player, Symmetry, Linearity, and Efficiency. Here, we introduce the Linearity, which will be used in subsequent analysis.

\begin{lemma}[Linearity of the Shapley value]
For any two utility functions $U_1,U_2$ and any $\alpha_1,\alpha_2\in\mathbb{R}$, $\phi_z(\alpha_1 U_1+\alpha_2 U_2) = \alpha_1\phi_z(U_1)+\alpha_2\phi_z(U_2).$
\end{lemma}

\noindent\textbf{In-Run Data Shapley.}
Classical Data Shapley is computationally expensive because it requires retraining models on many data subsets. In-Run Data Shapley avoids repeated retraining by attributing data contribution along a single realized training trajectory~\citep{wang2025data}.Consider a training run with checkpoints $\{w_t\}_{t=0}^{T}$, where $w_t$ denotes the model parameter at iteration $t$. 
At iteration $t$, let $B_t$ be the sampled mini-batch and let $z^{(\mathrm{val})}$ be a validation example. 
For any $S\subseteq B_t$, the \textit{local utility} is defined as the validation loss change induced by updating with $S$:
$U^{(t)}(S;z^{(\mathrm{val})})
:=
\ell(\widetilde w_{t+1}(S),z^{(\mathrm{val})})
-
\ell(w_t,z^{(\mathrm{val})}),$
where $\widetilde w_{t+1}(S)$ denotes the counterfactual next parameter obtained from the current state $w_t$ using the subset $S$. 
The global in-run utility is then defined by accumulating the local utilities over the training trajectory $U(S;z^{(\mathrm{val})}) =\sum_{t=0}^{T-1} U^{(t)}(S;z^{(\mathrm{val})}).$
By the linearity of the Shapley value, the corresponding In-Run Data Shapley value can be decomposed as $\phi_z(U) = \sum_{t=0}^{T-1} \phi_z\!\left(U^{(t)}\right),$
where $\phi_z(U^{(t)})=0$ if $z\notin B_t$. 
Thus, In-Run Data Shapley measures the cumulative contribution of a data point across the iterations in which it appears. 

\noindent\textbf{SGD Local Utility and Ghost Computation.}
Existing In-Run Data Shapley is derived under the update rule of stochastic gradient descent. 
For a subset $S\subseteq B_t$, the counterfactual SGD update is $\widetilde w_{t+1}(S)
= w_t-\eta_t\sum_{z\in S}\nabla_w \ell(w_t,z),$ where $\eta_t$ is the learning rate. 
Applying a first-order Taylor expansion of the validation loss around $w_t$ gives
{\small{\[
U_{\mathrm{SGD}}^{(t)}(S;z^{(\mathrm{val})})
\approx
-\eta_t
\left\langle
\nabla_w\ell(w_t,z^{(\mathrm{val})}),
\sum_{z\in S}\nabla_w\ell(w_t,z)
\right\rangle.
\]}}
This expression is additive over samples in $S$. 
Therefore, for any $z\in B_t$, the local Shapley value reduces to a gradient-gradient inner product $\phi_z(U_{\mathrm{SGD}}^{(t)}) \approx -\eta_t \langle \nabla_w\ell(w_t,z^{(\mathrm{val})}), \nabla_w\ell(w_t,z) \rangle,$ which enables efficient ghost-style computation without materializing per-sample gradients (more discussions in Section~\ref{sec:ghost_tech}).

\noindent\textbf{Adam Optimizer.}
Adam \cite{kingma2014adam} is a widely used adaptive optimizer that maintains exponential moving averages of both first-order ($m_t$) and second-order ($v_t$) gradient moments. At each step $t$, the update is given by:
\[
m_t(S)
= \beta_1 m_{t-1}+(1-\beta_1)g_t(S),\quad
v_t(S) = \beta_2 v_{t-1}+(1-\beta_2)g_t(S)^{\odot 2},
\]
where $g_t(S) = \nabla_w \ell(w_{t}; S)$ is the gradient of the training loss on the coalition $S$ and $\odot$ denotes element-wise multiplication. For clarity, we omit bias-correction terms or absorb them into the effective learning rate.  
The corresponding counterfactual Adam update is $\widetilde w_{t+1}(S)
= w_t - \eta_t \frac{m_t(S)}{\sqrt{v_t(S)}+\epsilon},$ where $\eta_t$ is the learning rate and $\epsilon$ is a small constant for numerical stability. Compared with SGD, Adam is stateful and nonlinear in the coalition gradient $g_t(S)$ due to historical moment estimates and coordinate-wise variance normalization. 
Therefore, the local utility under Adam no longer reduces to an additive sum of per-sample gradient inner products, preventing a direct application of standard In-Run Data Shapley and ghost dot-product computation.

\vspace{-0.1in}\section{In-Run Data Shapley for Adam Optimizer}\label{sec:method}\vspace{-0.1in}

In this section, we first empirically verify that data values are optimizer-dependent. 
We then develop an Adam-aware In-Run Data Shapley formulation and introduce an efficient Linearized Ghost computation for Adam-based training.

\vspace{-0.1in}\subsection{Optimizer Dependence of Data Value}
\label{sec:optimizer_dependence}
\begin{wrapfigure}{r}{0.45\linewidth}
    \centering
    \includegraphics[width=0.9\linewidth]{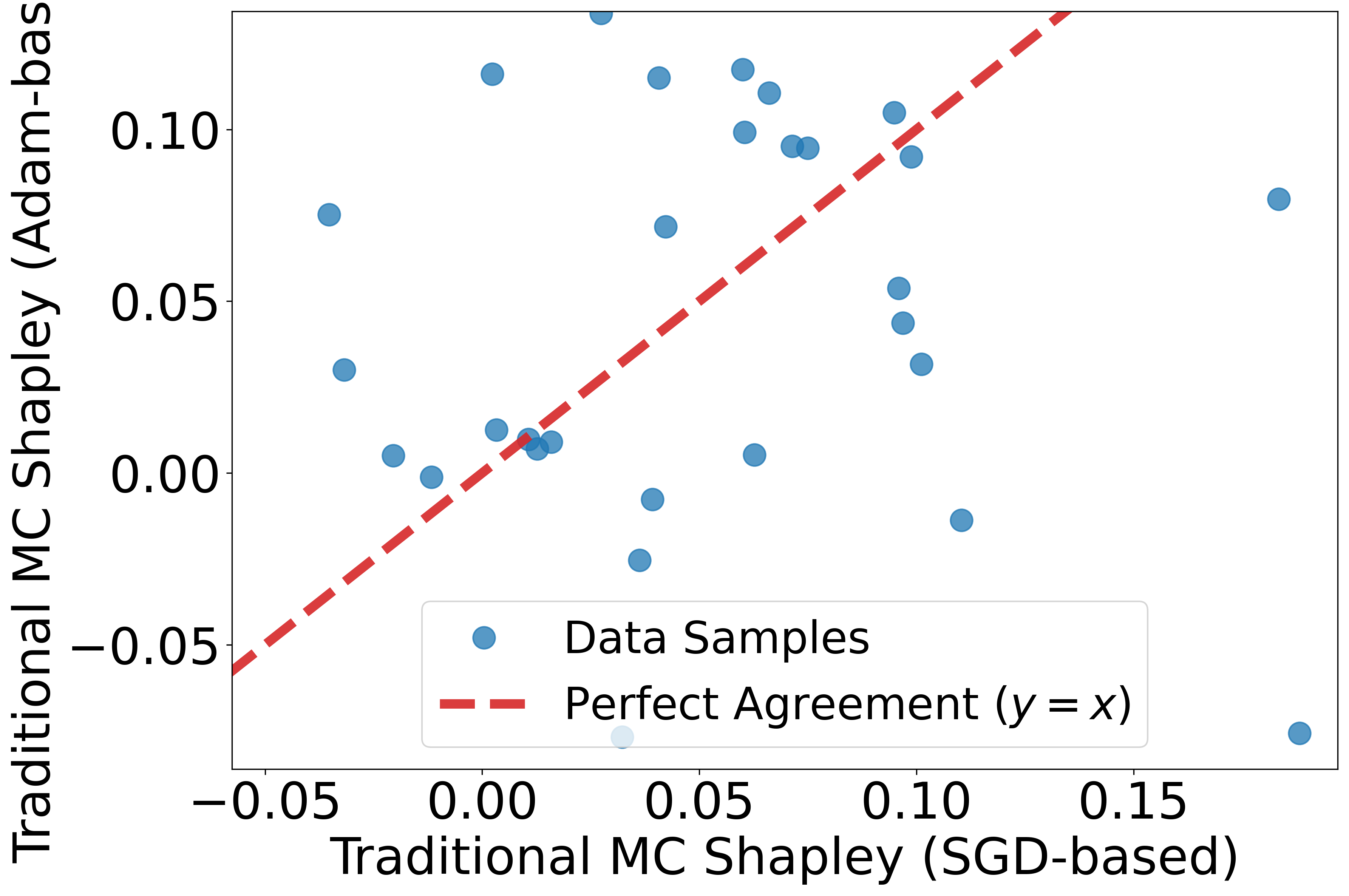}
    \caption{
    Comparison between SGD-based and Adam-based Data Shapley values
    (Pearson $R = 0.0579$, Spearman $\rho = 0.0465$).
    }
    \label{fig:sgd_vs_adam}
    \vspace{-1.2em}
\end{wrapfigure}
We begin by investigating a fundamental question: Are data Shapley values consistent across different optimization algorithms? To answer this, we computed the Shapley values for the same dataset and model architecture trained under two distinct regimes: SGD and Adam.
We use retraining-based Truncated Monte Carlo Shapley~\cite{ghorbani2019data} as a model-agnostic reference. TMC estimates the Shapley value by sampling random permutations of the training data and averaging the marginal utility change induced by each sample. Although computationally expensive, this retraining-based estimator allows us to compare how the choice of optimizer changes the induced data values without relying on our proposed approximation. (More experimental setup details refer to Appendix~\ref{sec:append_setup_opt}.)

As illustrated in Figure~\ref{fig:sgd_vs_adam}, attribution scores derived under the SGD trajectory diverge significantly from those computed under Adam. 
The scatter plot reveals a clear lack of agreement, with many samples deviating substantially from the identity line ($y=x$). 
Quantitatively, the Pearson correlation between SGD-based and Adam-based values is extremely low ($R=0.0579$), showing that the two optimizers assign substantially different attribution magnitudes. 
Moreover, the Spearman correlation is also very low ($\rho=0.0465$), indicating that the discrepancy is not merely a rescaling of value scores, but also changes the ranking of training samples. 
This inconsistency demonstrates that data value is not a static property intrinsic to the sample, but is fundamentally coupled to the optimization dynamics. 
Mechanistically, a sample that is highly influential under the linear updates of SGD may contribute differently under Adam, where the update direction is modulated by historical momentum and adaptive variance scaling. 
This suggests that relying on SGD-based attribution for adaptive optimizers can be misleading, motivating the need for an optimizer-aware framework.
In Section~\ref{sec:exp}, we further compare the proposed Adam-aware In-Run Shapley estimator with the standard SGD-based In-Run proxy. 
The results show that accounting for Adam dynamics substantially improves fidelity to explicit Adam marginal utility changes.


\vspace{-0.1in}\subsection{In-Run Data Shapley for Adam}

We now extend the in-run attribution framework to Adam. 
Under SGD, the local utility is additive over the coalition $S\subseteq B_t$, and the local Shapley value reduces to a gradient-gradient inner product:
\[
\phi_z(U_{\mathrm{SGD}}^{(t)})
\approx
-\eta_t
\left\langle
\nabla \ell(w_t,z^{(\mathrm{val})}),
\nabla \ell(w_t,z)
\right\rangle .
\]
This additive structure is the key reason why existing In-Run Data Shapley can be computed efficiently using ghost-style dot-products. 
However, Adam introduces optimizer states and adaptive variance normalization, so the update induced by a coalition is no longer a linear sum of per-sample gradients.

At iteration $t$, we condition on the realized Adam state $(w_t,m_{t-1},v_{t-1})$. 
For a coalition $S\subseteq B_t$, let $g_t(S):=\sum_{z\in S}\nabla \ell(w_t,z)$
denote the coalition gradient. Under the fixed-state local counterfactual, the previous moments $(m_{t-1},v_{t-1})$ are held fixed, while the current gradient is replaced by $g_t(S)$. 
Thus,
\[
m_t(S)=\beta_1m_{t-1}+(1-\beta_1)g_t(S),
\qquad
v_t(S)=\beta_2v_{t-1}+(1-\beta_2)g_t(S)^{\odot 2},
\]
and the corresponding counterfactual Adam update is $\widetilde w_{t+1}(S)
= w_t-\eta_t \frac{m_t(S)}{\sqrt{v_t(S)}+\epsilon}.$
We define the Adam local utility by the validation loss change after this counterfactual update:
$U_{\mathrm{Adam}}^{(t)}(S;z^{(\mathrm{val})})
:=
\ell(\widetilde w_{t+1}(S),z^{(\mathrm{val})})
-
\ell(w_t,z^{(\mathrm{val})}).$
Applying a first-order Taylor expansion of the validation loss around $w_t$ gives the following local approximation.

\begin{lemma}\label{lem:adam_sv}
For any iteration $t=0,\ldots,T-1$, under the fixed-state Adam local counterfactual, the first-order approximation of the Adam local utility is
{\small{\[
U_{\mathrm{Adam}}^{(t)}(S;z^{(\mathrm{val})})
\approx
-\eta_t
\left\langle
\nabla \ell(w_t,z^{(\mathrm{val})}),
\frac{m_t(S)}{\sqrt{v_t(S)}+\epsilon}
\right\rangle .
\]}}
\end{lemma}
Lemma~\ref{lem:adam_sv} shows that, under Adam, the local utility depends on the optimizer-induced update direction rather than the raw coalition gradient. 
This is fundamentally different from SGD: due to the variance term $v_t(S)$, the Adam update is nonlinear in the coalition gradient $g_t(S)$. 
As a result, the local utility is no longer additive over samples in $S$, and the corresponding Shapley value cannot be reduced to a standard gradient-gradient inner product. Therefore, the efficient ghost dot-product computation used in prior In-Run Data Shapley cannot be directly applied, which motivates the Linearized Ghost Approximation introduced next.

\vspace{-0.1in}\subsection{Efficient Computation via Linearized Ghost Dot-Product}
\label{sec:ghost_tech}

A major advantage of prior In-Run Data Shapley~\cite{wang2025data} is that, under SGD, the local contribution reduces to a gradient-gradient inner product. 
This structure enables the Ghost Dot-Product technique, which computes pairwise inner products between per-sample gradients and validation gradients without explicitly materializing full per-sample gradients.

Specifically, for a given layer $l$, let $a_i^{(l)}$ denote the input activation vector and $\delta_i^{(l)}$ denote the backpropagated error vector for sample $z_i$. 
The gradient of the weight matrix $W^{(l)}$ with respect to $z_i$ can be written as the outer product
$
g_i^{(l)}=\delta_i^{(l)}(a_i^{(l)})^\top.
$
Thus, the gradient dot-product between a training sample $z_i$ and a validation sample $z^{(\mathrm{val})}$ can be computed layer-wise as
\begin{equation*}
\begin{aligned}
\nabla \ell(w_t,z_i)^\top \nabla \ell(w_t,z^{(\mathrm{val})})
&=
\sum_l
\left\langle
\delta_i^{(l)}(a_i^{(l)})^\top,
\delta_{\mathrm{val}}^{(l)}(a_{\mathrm{val}}^{(l)})^\top
\right\rangle_F  =
\sum_l
\underbrace{(\delta_i^{(l)\top}\delta_{\mathrm{val}}^{(l)})}_{\text{error correlation}}
\cdot
\underbrace{(a_i^{(l)\top}a_{\mathrm{val}}^{(l)})}_{\text{activation correlation}}.
\end{aligned}
\end{equation*}
Therefore, the required dot-products can be obtained from activations and backpropagated errors, without instantiating per-sample gradients.
However, this standard ghost computation cannot be directly applied to Adam. 
From Lemma~\ref{lem:adam_sv}, the Adam local utility depends on the nonlinear update map $\Psi_t(g)
:=
\frac{\beta_1 m_{t-1}+(1-\beta_1)g}
{\sqrt{\beta_2 v_{t-1}+(1-\beta_2)g^{\odot 2}}+\epsilon},$
evaluated at the coalition gradient $g_t(S)$. 
Thus,
\[
U_{\mathrm{Adam}}^{(t)}(S;z^{(\mathrm{val})})
\approx
-\eta_t
\left\langle
\nabla \ell(w_t,z^{(\mathrm{val})}),
\Psi_t(g_t(S))
\right\rangle .
\]
Since $\Psi_t(\cdot)$ is nonlinear in $g_t(S)$, the marginal contribution of adding a sample no longer takes the standard gradient-gradient form.

To address this challenge, we introduce the \textbf{{Linearized Ghost Approximation}}. 
The key idea is to linearize $\Psi_t(\cdot)$ around a fixed reference gradient $\bar g_t$, which we take as the observed batch gradient $\Psi_t(g)
\approx \Psi_t(\bar g_t)+J_t(g-\bar g_t), J_t:=\left.\frac{\partial \Psi_t(g)}{\partial g}\right|_{g=\bar g_t}.$ 
Substituting this approximation into the Adam local utility gives an affine function of the coalition gradient $g_t(S)$. 
The terms independent of $S$ vanish under the Shapley value, and the remaining term restores an additive structure over samples in $S$. 
Consequently, the local contribution can again be written as a bilinear form between an Adam-aware validation direction and a per-sample gradient, making ghost-style computation applicable.
\begin{theorem}\label{thm:adam_sv_approx}
Let $\bar g_t$ be a fixed reference gradient and 
$J_t=\left.\frac{\partial \Psi_t(g)}{\partial g}\right|_{g=\bar g_t}$.
Under the Linearized Ghost Approximation $\Psi_t(g_t(S))
\approx \Psi_t(\bar g_t)+J_t (g_t(S)-\bar g_t),$ the resulting linearized Adam local utility satisfies, for any $z\in B_t$, $\phi_z (\widetilde U_{\mathrm{Adam}}^{(t)})
= -\eta_t \langle \nabla \ell(w_t,z^{(\mathrm{val})}), J_t\nabla \ell(w_t,z) \rangle,$
where $\widetilde U_{\mathrm{Adam}}^{(t)}$ denotes the local utility obtained by substituting the above linearization into $U_{\mathrm{Adam}}^{(t)}$.
Therefore, the global Adam-aware In-Run Data Shapley value under the linearized utility is
{\[
\phi_z(\widetilde U_{\mathrm{Adam}})
=
-\sum_{t:z\in B_t}
\eta_t
\left\langle
\nabla \ell(w_t,z^{(\mathrm{val})}),
J_t\nabla \ell(w_t,z)
\right\rangle .
\]}
\end{theorem}
Theorem~\ref{thm:adam_sv_approx} shows that the Shapley value of the linearized Adam local utility recovers a ghost-compatible dot-product form.  Compared with SGD-based In-Run Shapley, the validation direction is now modified by the Adam update Jacobian $J_t$.  Equivalently, $\langle
\nabla \ell(w_t,z^{(\mathrm{val})}), J_t\nabla \ell(w_t,z) \rangle = \langle
J_t^\top\nabla \ell(w_t,z^{(\mathrm{val})}), \nabla \ell(w_t,z) \rangle .$
Thus, Adam-aware attribution replaces the raw validation gradient used by SGD-based In-Run Shapley with an optimizer-aware validation direction. 
This transformation is generally not a scalar rescaling, and therefore can change the ranking of training samples.

\vspace{-0.1in}\section{Experiments}\vspace{-0.1in}
\label{sec:exp}
In this section, we conduct a comprehensive empirical evaluation to demonstrate the effectiveness and efficiency of our Adam-aware In-Run Data Shapley from three aspects:
(1) \textbf{Approximation Fidelity}, we test whether the proposed estimator faithfully matches exact local Shapley values under the fixed-state Adam utility.
(2) \textbf{Computational Efficiency}, we evaluate the computational efficiency of the Linearized Ghost Approximation. 
(3) \textbf{Practical Effectiveness}, we study whether Adam-aware attribution is useful for downstream rank-based tasks, including semantic source identification and data pruning.

\begin{figure*}[t]
    \centering
    \begin{subfigure}[t]{0.32\linewidth}
        \centering
        \includegraphics[width=\linewidth]{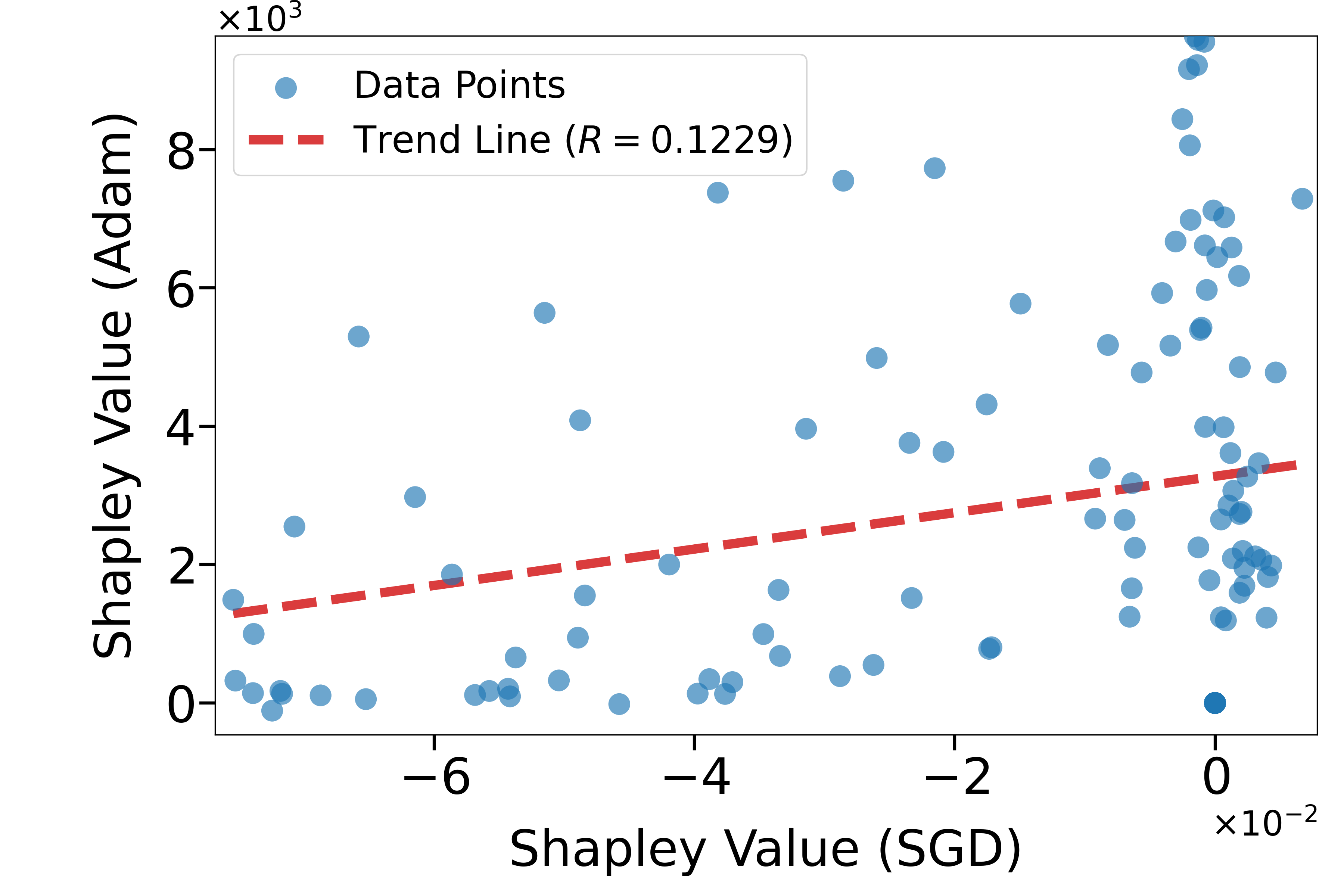}
        \caption{}
        \label{fig:ap_rq3}
    \end{subfigure}
    \hfill
    \begin{subfigure}[t]{0.32\linewidth}
        \centering
        \includegraphics[width=\linewidth]{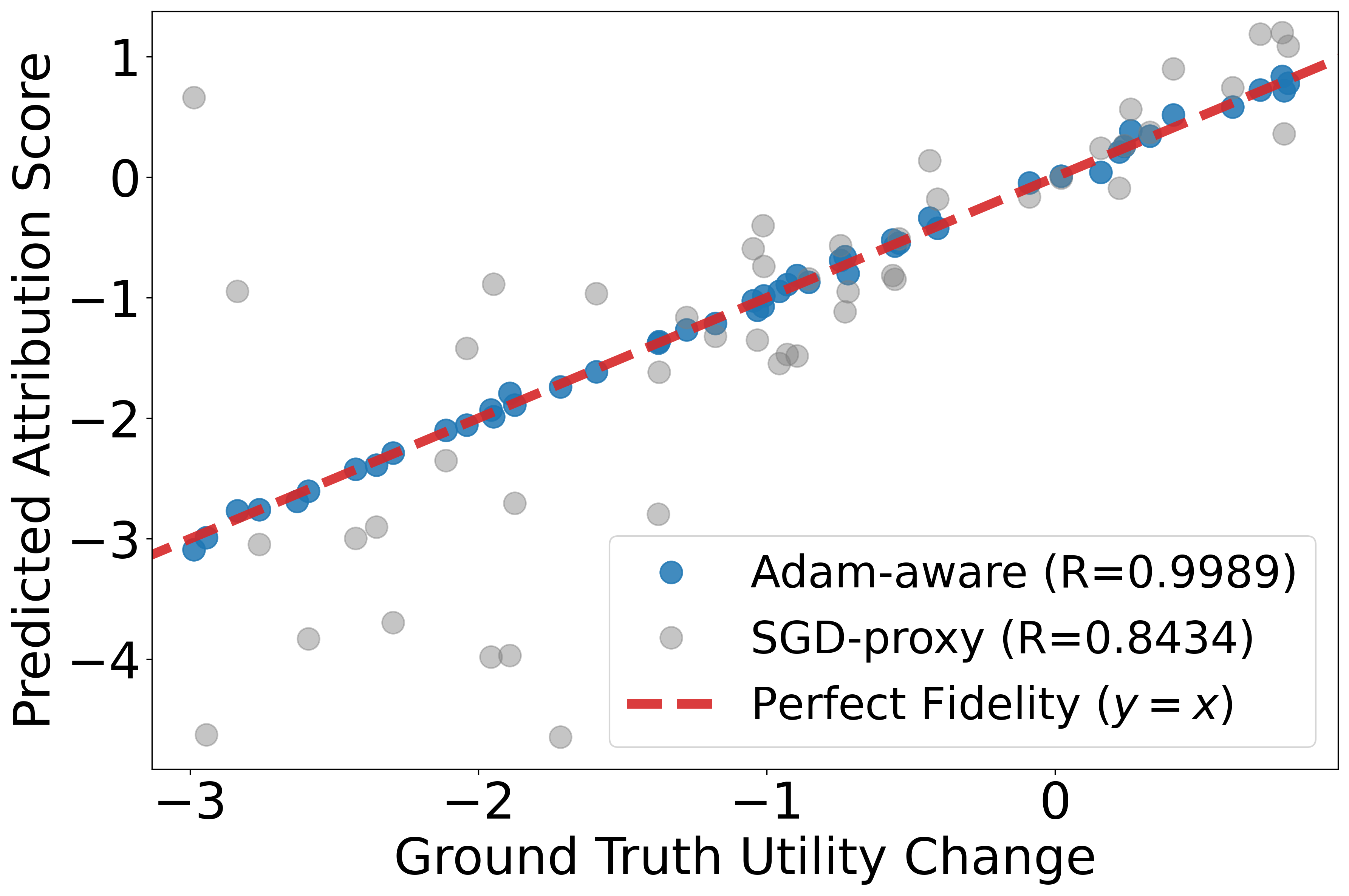}
        \caption{}
        \label{fig:ap_rq2}
    \end{subfigure}
    \hfill
    \begin{subfigure}[t]{0.32\linewidth}
        \centering
        \includegraphics[width=\linewidth]{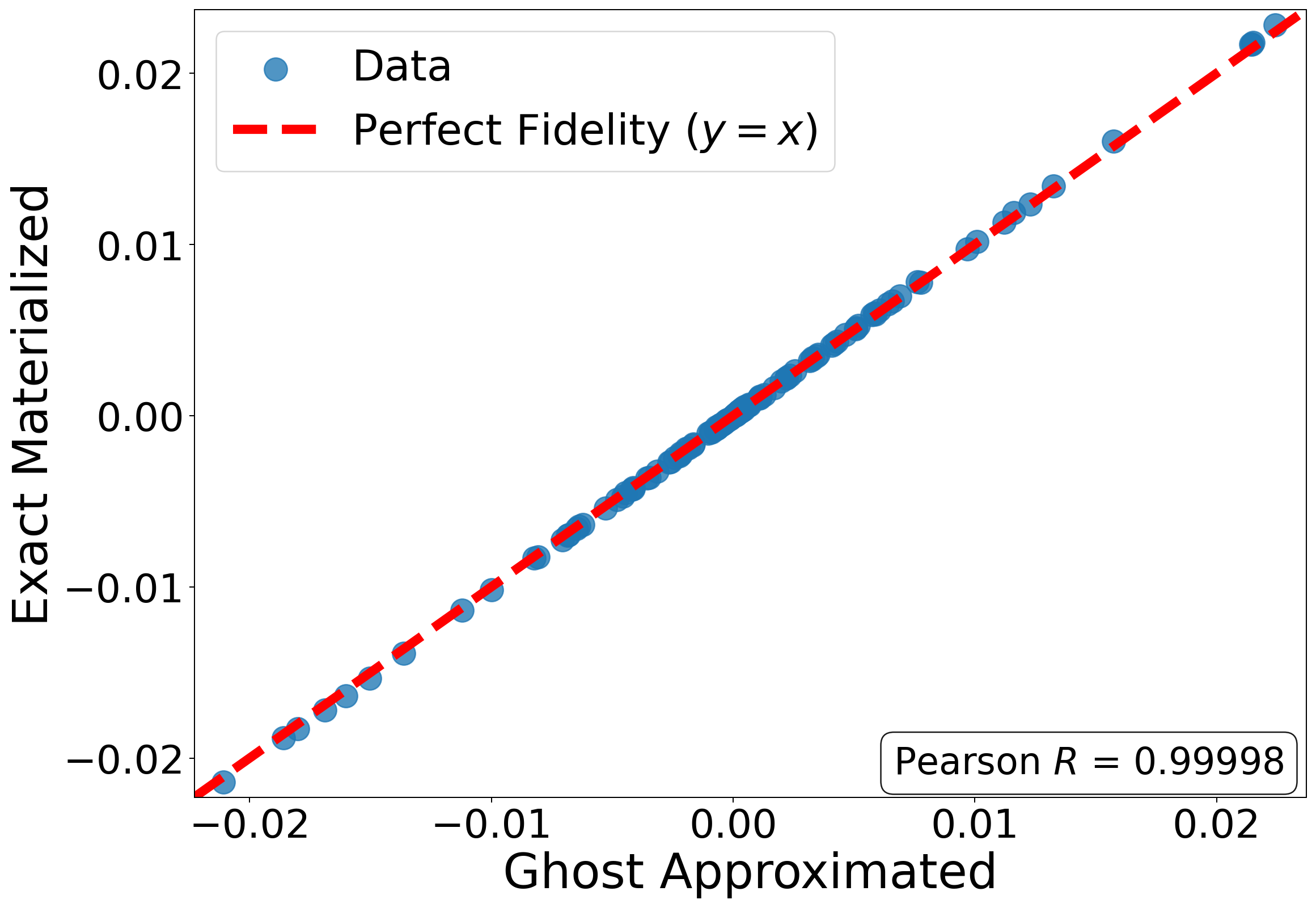}
        \caption{}
        \label{fig:ghost}
    \end{subfigure}

    \caption{
    \textbf{Optimizer dependence and fidelity.}
    {(a)} Accumulated SGD-proxy and Adam-aware scores differ along the same Adam trajectory.
    {(b)} Adam-aware scores better match exact local Shapley values than the SGD proxy ($R=0.9992$ vs. $0.8434$).
    {(c)} Linearized Ghost closely matches exact materialized Adam-aware scores ($R\approx0.999984$).
    }
    \label{fig:optimizer_shapley}
    \vspace{-0.1in}
\end{figure*}

\vspace{-0.1in}\subsection{Fidelity Comparison}\vspace{-0.1in}

In this subsection, we evaluate whether the proposed Adam-aware approximation faithfully captures the data value induced by Adam updates. 
We separate two levels of evaluation. 
First, we compare accumulated In-Run attribution scores computed along the same Adam trajectory using the SGD-based formula and the Adam-aware formula. 
Second, at a fixed training iteration, we compare both estimators against exact local Shapley values computed on the current mini-batch under the fixed-state Adam local utility. 
This distinction is important: the first protocol studies optimizer-induced differences in accumulated data values, while the second measures local approximation fidelity.

\noindent\textbf{Experimental Setup.}
We compare our \textbf{Adam-aware approximation} with the \textbf{SGD-based In-Run proxy}~\citep{wang2025data}. 
We consider three complementary protocols:
(i) comparing accumulated In-Run scores computed on the same Adam trajectory using the SGD-based and Adam-aware formulas;
(ii) measuring local fidelity against exact Shapley values computed for the current mini-batch at a fixed training iteration; and
(iii) evaluating robustness across learning rates $\eta\in[10^{-7},10^{-3}]$. 
For protocol (i), we report the correlation between accumulated SGD-proxy and Adam-aware scores. 
For protocols (ii) and (iii), we report Pearson correlation ($R$) between predicted local scores and exact local Shapley values.

\begin{wrapfigure}{r}{0.45\linewidth}
    \centering
    \includegraphics[width=0.8\linewidth]{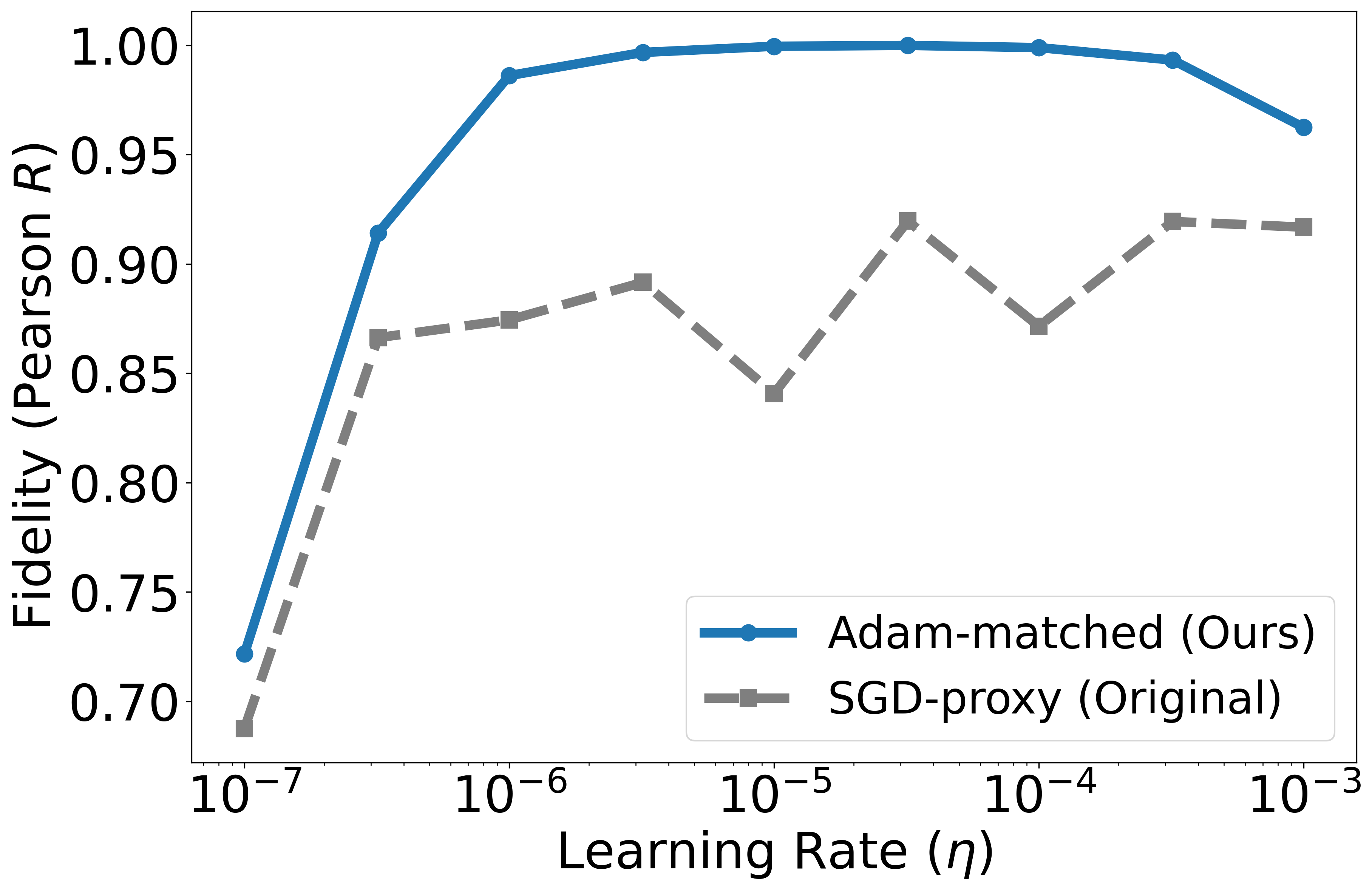}
    \caption{
    {Fidelity across learning rates.}
    }
    \vspace{-0.1in}
    \label{fig:rq2_fidelity_trend}
\end{wrapfigure}

\noindent\textbf{Optimizer Dependence in Accumulated Values.}
Figure~\ref{fig:ap_rq3} compares accumulated In-Run scores computed along the same Adam training trajectory using the SGD-based formula and the Adam-aware formula. 
The weak alignment indicates that replacing Adam's optimizer-induced update direction with the raw SGD gradient can substantially change the resulting data values. 
Thus, optimizer mismatch affects not only the numerical scale of attribution scores, but also their accumulated ranking across training.

\noindent\textbf{Fidelity to Exact Local Shapley Values.}
Figure~\ref{fig:ap_rq2} compares both methods against exact local Shapley values computed on the current mini-batch under the fixed-state Adam local utility. 
The Adam-aware approximation achieves near-perfect fidelity ($R=0.9992$), while the SGD-based proxy has lower correlation ($R=0.8434$). 
This shows that the Adam-aware update direction more accurately tracks the optimizer-induced local utility. 
The same trend holds across learning rates in Figure~\ref{fig:rq2_fidelity_trend}: Adam-aware attribution remains highly stable with $R>0.96$, whereas the SGD-proxy is less stable and degrades under some learning rates.

\noindent\textbf{Validity of Linearized Ghost.}
Finally, we evaluate whether the Linearized Ghost Approximation preserves the Adam-aware attribution signal. 
Figure~\ref{fig:ghost} compares ghost-approximated scores with exact materialized Adam-aware scores. 
The near-perfect agreement ($R\approx0.999984$) indicates that the ghost implementation introduces negligible additional approximation error relative to the explicit Adam-aware computation. 
Thus, Linearized Ghost provides a scalable implementation while preserving the fidelity of Adam-aware attribution.

\vspace{-0.1in}\subsection{Computational Efficiency via Linearized Ghost Approximation}
\label{sec:efficiency}\vspace{-0.1in}

A critical contribution of this work is the \textit{Linearized Ghost Approximation}, which makes Adam-aware attribution feasible without materializing per-sample gradients. In this subsection, we evaluate whether this approximation reduces the computational bottleneck of direct Adam-aware attribution.

\noindent\textbf{Experimental Setup.}
We benchmark runtime and memory on GPT-2 Small (124M parameters) trained with AdamW. 
All experiments are conducted on a single NVIDIA A100 (80GB) GPU, using batch size $16$ and sequence length $1024$. 
We compare three configurations: 
(i) \textbf{Standard AdamW}, training without attribution; 
(ii) \textbf{Adam-Ghost (Ours)}, our Adam-aware In-Run implementation with Linearized Ghost; and 
(iii) \textbf{Adam-Direct (Naive)}, which explicitly materializes per-sample gradients for Adam-aware attribution.

\noindent\textbf{Results.}
As shown in Table~\ref{tab:efficiency}, Adam-Ghost achieves $112.21$ samples/sec, retaining $43.9\%$ of the throughput of standard AdamW training. 
This shows that Adam-aware attribution introduces non-negligible overhead relative to ordinary training, mainly due to collecting attribution signals and applying the Adam-aware linearized weighting during training. 
However, Adam-Ghost is substantially more efficient than the naive Adam-Direct implementation, which achieves only $50.06$ samples/sec. 
Thus, Linearized Ghost improves throughput by $2.24\times$ over Adam-Direct by replacing explicit per-sample Adam-aware computations with ghost-style aggregation.

\begin{wraptable}{r}{0.50\linewidth}
\vspace{-1.0em}
\centering
\caption{\textbf{Efficiency on GPT-2 Small} ($BS=16$, $Seq=1024$). Adam-Ghost is more efficient than Adam-Direct.}
\label{tab:efficiency}
\vspace{-0.6em}
\footnotesize
\setlength{\tabcolsep}{4pt}
\begin{tabular}{lcc}
\toprule
\textbf{Method} & \textbf{SPS} & \textbf{Mem. (MB)} \\
\midrule
Standard & \textbf{255.62} & \textbf{1437.52} \\
Adam-Ghost & 112.21 & 3232.09 \\
Adam-Direct & 50.06 & 10540.37 \\
\bottomrule
\end{tabular}
\vspace{-0.8em}
\end{wraptable}

The memory results show the same trend. 
Adam-Ghost uses $3232.09$ MB of peak memory, which is higher than standard training ($1437.52$ MB) but much lower than Adam-Direct ($10540.37$ MB). 
This corresponds to a $69.3\%$ reduction in peak memory compared with Adam-Direct. 
Overall, these results show that Linearized Ghost does not eliminate all overhead relative to standard training, but it makes Adam-aware attribution substantially more practical than explicit per-sample computation.

\begin{table*}[t]
    \centering
    
    \begin{minipage}[t]{0.45\linewidth}
        \begin{tcolorbox}[
            colframe=ubBlue,
            colbacktitle=ubBlue,
            colback=white,
            title=\textbf{Original Wikipedia Corpus},
            arc=2mm, boxrule=1pt,
            fonttitle=\bfseries
        ]
        \small
        The arsenal was briefly seized once more by Joseph Brooks loyalists during the Brooks-Baxter War of 1874.
        \end{tcolorbox}
    \end{minipage}
    \hspace{0.3cm}
    \begin{minipage}[t]{0.45\linewidth}
        \begin{tcolorbox}[
            colframe=kaustOrange,
            colbacktitle=kaustOrange,
            colback=white,
            title=\textbf{Synthetic ``Similar topic" Query},
            arc=2mm, boxrule=1pt,
            fonttitle=\bfseries
        ]
        \small
        A historical dispute involved temporary takeover of a military installation by opposing forces.
        \end{tcolorbox}
    \end{minipage}
    
    \vspace{0.2cm}

    \resizebox{0.95\linewidth}{!}{
        \begin{tabular}{lcccc}
            \toprule
            \textbf{Scenario} 
            & \textbf{BM25} 
            & \textbf{IF-Proxy} 
            & \textbf{In-Run Data Shapley via SGD} 
            & \textbf{In-Run Data Shapley via Adam} \\
            \midrule
            Exact                 & \textbf{1.00} & 5.07    & 1.05    & 1.55 \\
            Partial               & \textbf{1.00} & 9.87    & 1.25    & 1.65 \\
            Paraphrase            & \textbf{1.21} & 245.41  & 9.48    & 6.51 \\
            Significant rewrite   & \textbf{2.05} & 216.23  & 454.60  & 32.10 \\
            Similar topic         & 2643.50       & 2855.13 & 1612.50 & \textbf{753.25} \\
            \bottomrule
        \end{tabular}
    }

    \caption{\textbf{Semantic Source Identification.}
\textbf{Top}: an original Wikipedia corpus and a synthetic ``Similar topic'' query.
\textbf{Bottom}: average rank of the true source among $10\,\mathrm{k}$ candidates; lower is better.
BM25 is a lexical baseline, and IF-Proxy is a post-training gradient-dot-product proxy.
Adam-based In-Run Data Shapley improves over SGD-based In-Run Data Shapley on challenging semantic variants, especially significant rewrite and similar topic.}
    \label{tab:semantic_robustness}
    \vspace{-0.2in}
\end{table*}

\vspace{-0.1in}\subsection{Practical Effectiveness}\vspace{-0.1in}
We evaluate the practical utility of our method on two rank-based data attribution tasks: semantic source identification and data pruning.

\vspace{-0.1in}\subsubsection{Semantic Source Identification}

\noindent\textbf{Experimental Setup.}
We first evaluate whether Adam-aware attribution can recover semantically related source examples from an Adam-trained language model. 
We train DistilGPT-2 (82M)~\citep{sanh2019distilbert,radford2019gpt2} on WikiText-2~\citep{merity2017pointer}. 
For efficiency, we use the first $10\,\mathrm{k}$ training lines ($\sim 350\,\mathrm{k}$ tokens) for training and reserve the remaining data for validation and testing. 
We compare four methods: BM25, IF-Proxy, SGD-aware In-Run Data Shapley, and Adam-aware In-Run Data Shapley.

\noindent\textbf{Semantic Perturbation Protocol and Evaluation Metric.}
For each trial, we randomly select a training example $z^{\ast}$ as the source sample and construct validation queries with increasing semantic distance:
(i) \emph{Exact},
(ii) \emph{Partial},
(iii) \emph{Paraphrase},
(iv) \emph{Significant rewrite}, and
(v) \emph{Similar topic}.
The paraphrase and significant-rewrite variants are generated by a paraphrasing model and filtered using lexical-overlap constraints. 
Each generated query is treated as a validation input $z^{(\mathrm{val})}$.
For each validation query $z^{(\mathrm{val})}$, we compute an attribution or retrieval score $\phi_i$ for every training sample $z_i$. 
We then rank all training samples and report the average rank of the true source sample $z^{\ast}$ across multiple queries and random trials. 
A lower rank indicates better source identification.

\noindent\textbf{Results and Analysis.}
Table~\ref{tab:semantic_robustness} reports the average rank of the true source sample. 
BM25 performs well when lexical overlap is preserved, especially in the exact, partial, and paraphrase settings, which is expected because BM25 directly relies on surface-form matching. 
Thus, BM25 should be viewed as a lexical retrieval baseline rather than an optimizer-aware attribution method.

The main comparison is between SGD-aware and Adam-aware In-Run Data Shapley. 
Adam-aware attribution performs better on the more semantically challenging variants. 
For paraphrased queries, it improves the average source rank from $9.48$ to $6.51$. 
Under significant rewriting, the improvement is much larger, from $454.60$ to $32.10$. 
For similar-topic queries, Adam-aware attribution further improves the rank from $1612.50$ to $753.25$. 
These results suggest that Adam-aware In-Run Data Shapley captures optimizer-dependent training signals that are missed by the SGD-style proxy.

The IF-Proxy baseline provides a post-training gradient-dot-product signal, but it does not account for the full training trajectory or Adam's adaptive moment states. 
As a result, it performs substantially worse than Adam-aware In-Run Data Shapley on the semantic variants. 
Overall, these results provide rank-based evidence that attribution under Adam training benefits from explicitly modeling Adam's optimizer dynamics. 
Additional results are provided in Appendix~\ref{app:semantic_1k}.


\vspace{-0.1in}\subsubsection{Data Pruning on SST-2}

\noindent\textbf{Experimental Setup.}
We evaluate whether optimizer-matched In-Run Shapley scores can guide data pruning. 
We train DistilBERT~\citep{sanh2019distilbert} on a subset of $10{,}000$ SST-2 training samples~\citep{socher2013recursive} using AdamW with learning rate $2\times10^{-5}$, batch size $16$, and $5$ epochs. 
During training, we compute Adam-aware In-Run Shapley scores with the Linearized Ghost Approximation. 
We then remove the top-ranked, bottom-ranked, or randomly selected $10\%$--$30\%$ training samples, retrain the model on the remaining data, and evaluate accuracy on the full validation set of $872$ examples. 
All results are averaged over three seeds. 
Random pruning serves as a budget-matched control, separating the effect of data selection from the effect of reducing the training set size.

To compare pruning behavior under another optimizer, we repeat the same protocol under SGD dynamics. 
Specifically, we compute SGD-aware In-Run Shapley scores and retrain DistilBERT using SGD with momentum $0.9$, learning rate $3\times10^{-4}$, and a linear warmup over $6\%$ of the total training steps. 
All other settings, including data splits and pruning ratios, are kept fixed.

\begin{table}[t]
\centering
\captionsetup{font=small}
\caption{
\textbf{Data pruning on SST-2 with DistilBERT.}
Validation accuracy after removing bottom-, random-, or top-ranked samples.
}
\label{tab:sst2_pruning}
\small
\setlength{\tabcolsep}{7pt}
\begin{tabular}{lccc ccc}
\toprule
& \multicolumn{3}{c}{\textbf{AdamW}} 
& \multicolumn{3}{c}{\textbf{SGD}} \\
\cmidrule(lr){2-4}\cmidrule(lr){5-7}
\textbf{Ratio} 
& \textbf{Bottom} & \textbf{Random} & \textbf{Top} 
& \textbf{Bottom} & \textbf{Random} & \textbf{Top} \\
\midrule
10\% & \textbf{0.8838} & 0.8713 & 0.8685 & \textbf{0.8392} & 0.8131 & 0.8119 \\
20\% & \textbf{0.8826} & 0.8752 & 0.8639 & \textbf{0.8228} & 0.7706 & 0.7649 \\
30\% & \textbf{0.8876} & 0.8739 & 0.8532 & \textbf{0.7117} & 0.6760 & 0.6383 \\
\bottomrule
\end{tabular}
\vspace{-0.2in}
\end{table}

\noindent\textbf{Results and Analysis.}
Table~\ref{tab:sst2_pruning} reports validation accuracy after pruning different fractions of the training set. 
Under AdamW, removing the \emph{bottom}-ranked samples according to Adam-aware In-Run Shapley consistently gives the best performance. 
At pruning ratios of $10\%$, $20\%$, and $30\%$, bottom pruning achieves $0.8838 / 0.8826 / 0.8876$ accuracy, outperforming random pruning ($0.8713 / 0.8752 / 0.8739$) and top pruning ($0.8685 / 0.8639 / 0.8532$). 
This shows that Adam-aware attribution identifies low-value samples whose removal is more beneficial than equal-budget random pruning.
Under SGD, bottom pruning also outperforms random and top pruning, but the overall performance is substantially lower and degrades more sharply as the pruning ratio increases. 
At $30\%$ pruning, even the best SGD condition drops to $0.7117$, while top pruning decreases to $0.6383$. 
Across both optimizers, top pruning is consistently harmful, which is expected because it removes samples assigned high positive contribution by the corresponding attribution score. 
Overall, these results show that In-Run Shapley scores provide meaningful rankings for data pruning. 
They also reinforce that data values should be interpreted relative to the optimizer and training trajectory under which they are computed.

\vspace{-0.1in}\section{Conclusion}\vspace{-0.1in}
\label{sec:conclusion}

In this work, we propose \emph{Adam-Aware In-Run Data Shapley}, an optimizer-consistent extension of In-Run Data Shapley for Adam training. 
Our analysis and experiments show that data values are coupled with the realized optimization trajectory, and that SGD-based in-run attribution can be misleading for Adam-trained models. 
By defining a fixed-state Adam local utility and introducing a Linearized Ghost Approximation, our method accounts for Adam's momentum and variance normalization while remaining substantially more efficient than direct per-sample computation. 
Experiments demonstrate near-perfect fidelity to explicit Adam marginal utility changes and improved performance on rank-based data attribution tasks, including semantic source identification and data pruning.

\section*{Acknowledgments}
We thank the anonymous ICML reviewers of an earlier version of this work for their thoughtful feedback, which helped improve the presentation, technical discussion, experimental interpretation, and related work coverage.

\bibliography{icml/conference}

@article{agarwal2017second,
  title={Second-order stochastic optimization for machine learning in linear time},
  author={Agarwal, Naman and Bullins, Brian and Hazan, Elad},
  journal={Journal of Machine Learning Research},
  volume={18},
  number={116},
  pages={1--40},
  year={2017}
}

@article{kwon2023datainf,
  title={Datainf: Efficiently estimating data influence in lora-tuned llms and diffusion models},
  author={Kwon, Yongchan and Wu, Eric and Wu, Kevin and Zou, James},
  journal={arXiv preprint arXiv:2310.00902},
  year={2023}
}

@article{zhou2024hyperinf,
  title={HyperINF: Unleashing the HyperPower of the Schulz's Method for Data Influence Estimation},
  author={Zhou, Xinyu and Fan, Simin and Jaggi, Martin},
  journal={arXiv preprint arXiv:2410.05090},
  year={2024}
}

@article{bae2024training,
  title={Training data attribution via approximate unrolled differentiation},
  author={Bae, Juhan and Lin, Wu and Lorraine, Jonathan and Grosse, Roger},
  journal={arXiv preprint arXiv:2405.12186},
  year={2024}
}

@article{ilyas2025magic,
  title={Magic: Near-optimal data attribution for deep learning},
  author={Ilyas, Andrew and Engstrom, Logan},
  journal={arXiv preprint arXiv:2504.16430},
  year={2025}
}

@article{gao2020pile,
  title={The pile: An 800gb dataset of diverse text for language modeling},
  author={Gao, Leo and Biderman, Stella and Black, Sid and Golding, Laurence and Hoppe, Travis and Foster, Charles and Phang, Jason and He, Horace and Thite, Anish and Nabeshima, Noa and others},
  journal={arXiv preprint arXiv:2101.00027},
  year={2020}
}

@article{loshchilov2017decoupled,
  title={Decoupled weight decay regularization},
  author={Loshchilov, Ilya and Hutter, Frank},
  journal={arXiv preprint arXiv:1711.05101},
  year={2017}
}

@article{kingma2014adam,
  title={Adam: A method for stochastic optimization},
  author={Kingma, Diederik P and Ba, Jimmy},
  journal={arXiv preprint arXiv:1412.6980},
  year={2014}
}

@article{raffel2020exploring,
  title={Exploring the limits of transfer learning with a unified text-to-text transformer},
  author={Raffel, Colin and Shazeer, Noam and Roberts, Adam and Lee, Katherine and Narang, Sharan and Matena, Michael and Zhou, Yanqi and Li, Wei and Liu, Peter J},
  journal={Journal of machine learning research},
  volume={21},
  number={140},
  pages={1--67},
  year={2020}
}

@article{DBLP:journals/corr/abs-2501-15963,
  author       = {Chenyang Ren and
                  Huanyi Xie and
                  Shu Yang and
                  Meng Ding and
                  Lijie Hu and
                  Di Wang},
  title        = {Evaluating Data Influence in Meta Learning},
  journal      = {CoRR},
  volume       = {abs/2501.15963},
  year         = {2025}
}

@article{DBLP:journals/corr/abs-2507-04059,
  author       = {Chenyang Ren and
                  Yifan Jia and
                  Huanyi Xie and
                  Zhaobin Xu and
                  Tianxing Wei and
                  Liangyu Wang and
                  Lijie Hu and
                  Di Wang},
  title        = {Attributing Data for Sharpness-Aware Minimization},
  journal      = {CoRR},
  volume       = {abs/2507.04059},
  year         = {2025}
}

@misc{hu2025dissectingrepresentationmisalignmentcontrastive,
      title={Dissecting Representation Misalignment in Contrastive Learning via Influence Function}, 
      author={Lijie Hu and Chenyang Ren and Huanyi Xie and Khouloud Saadi and Shu Yang and Zhen Tan and Jingfeng Zhang and Di Wang},
      year={2025},
      eprint={2411.11667},
      archivePrefix={arXiv},
      primaryClass={cs.LG},
      url={https://arxiv.org/abs/2411.11667}, 
}

@incollection{
shapley1953value,
title        = {A Value for n-Person Games},
author       = {Shapley, Lloyd S.},
booktitle    = {Contributions to the Theory of Games, Volume II},
editor       = {Kuhn, H. W. and Tucker, A. W.},
pages        = {307--317},
year         = {1953},
publisher    = {Princeton University Press},
address      = {Princeton, NJ}
}

@article{Jia2019EfficientTD,
  title={Efficient Task-Specific Data Valuation for Nearest Neighbor Algorithms},
  author={Ruoxi Jia and David Dao and Boxin Wang and Frances Ann Hubis and Nezihe Merve G{\"u}rel and Bo Li and Ce Zhang and Costas J. Spanos and Dawn Xiaodong Song},
  journal={ArXiv},
  year={2019},
  volume={abs/1908.08619},
  url={https://api.semanticscholar.org/CorpusID:263794198}
}

@inproceedings{
wang2025data,
title={Data Shapley in One Training Run},
author={Jiachen T. Wang and Prateek Mittal and Dawn Song and Ruoxi Jia},
booktitle={The Thirteenth International Conference on Learning Representations},
year={2025},
url={https://openreview.net/forum?id=HD6bWcj87Y}
}

@inproceedings{ghorbani2019data,
  title={Data Shapley: Equitable valuation of data for machine learning},
  author={Ghorbani, Amirata and Zou, James},
  booktitle={International Conference on Machine Learning},
  year={2019}
}

@inproceedings{jia2019towards,
  title={Towards efficient data valuation based on the shapley value},
  author={Jia, Ruoxi and Dao, David and Wang, Boxin and Hubis, Frances Ann and Hynes, Nick and G{\"u}rel, Nezihe Merve and Li, Bo and Zhang, Ce and Song, Dawn and Spanos, Costas J},
  booktitle={The 22nd international conference on artificial intelligence and statistics},
  pages={1167--1176},
  year={2019},
  organization={PMLR}
}

@article{pruthi2020estimating,
  title={Estimating training data influence by tracing gradient descent},
  author={Pruthi, Garima and Liu, Frederick and Kale, Satyen and Sundararajan, Mukund},
  journal={Advances in Neural Information Processing Systems},
  volume={33},
  pages={19920--19930},
  year={2020}
}

@article{basu2020influence,
  title={Influence functions in deep learning are fragile},
  author={Basu, Samyadeep and Pope, Philip and Feizi, Soheil},
  journal={arXiv preprint arXiv:2006.14651},
  year={2020}
}

@inproceedings{koh2017understanding,
  title={Understanding black-box predictions via influence functions},
  author={Koh, Pang Wei and Liang, Percy},
  booktitle={International conference on machine learning},
  pages={1885--1894},
  year={2017},
  organization={PMLR}
}

@article{paul2021deep,
  title={Deep learning on a data diet: Finding important examples early in training},
  author={Paul, Mansheej and Ganguli, Surya and Dziugaite, Gintare Karolina},
  journal={Advances in neural information processing systems},
  volume={34},
  pages={20596--20607},
  year={2021}
}

@article{kwon2021beta,
  title={Beta shapley: a unified and noise-reduced data valuation framework for machine learning},
  author={Kwon, Yongchan and Zou, James},
  journal={arXiv preprint arXiv:2110.14049},
  year={2021}
}

@inproceedings{wang2023data,
  title={Data banzhaf: A robust data valuation framework for machine learning},
  author={Wang, Jiachen T and Jia, Ruoxi},
  booktitle={International Conference on Artificial Intelligence and Statistics},
  pages={6388--6421},
  year={2023},
  organization={PMLR}
}

@article{li2023robust,
  title={Robust data valuation with weighted banzhaf values},
  author={Li, Weida and Yu, Yaoliang},
  journal={Advances in Neural Information Processing Systems},
  volume={36},
  pages={60349--60383},
  year={2023}
}

@inproceedings{ilyas2022datamodels,
  title = {Datamodels: Predicting Predictions from Training Data},
  author = {Andrew Ilyas and Sung Min Park and Logan Engstrom and Guillaume Leclerc and Aleksander Madry},
  booktitle = {International Conference on Machine Learning (ICML)},
  year = {2022}
}

@article{illes2019estimation,
  title={Estimation of the Shapley value by ergodic sampling},
  author={Ill{\'e}s, Ferenc and Ker{\'e}nyi, P{\'e}ter},
  journal={arXiv preprint arXiv:1906.05224},
  year={2019}
}

@inproceedings{okhrati2021multilinear,
  title={A multilinear sampling algorithm to estimate shapley values},
  author={Okhrati, Ramin and Lipani, Aldo},
  booktitle={2020 25th International Conference on Pattern Recognition (ICPR)},
  pages={7992--7999},
  year={2021},
  organization={IEEE}
}

@inproceedings{burgess2021approximating,
  title={Approximating the Shapley Value Using Stratified Empirical Bernstein Sampling.},
  author={Burgess, Mark Alexander and Chapman, Archie C},
  booktitle={IJCAI},
  pages={73--81},
  year={2021}
}

@inproceedings{lin2022measuring,
  title={Measuring the effect of training data on deep learning predictions via randomized experiments},
  author={Lin, Jinkun and Zhang, Anqi and L{\'e}cuyer, Mathias and Li, Jinyang and Panda, Aurojit and Sen, Siddhartha},
  booktitle={International Conference on Machine Learning},
  pages={13468--13504},
  year={2022},
  organization={PMLR}
}

@article{mitchell2022sampling,
  title={Sampling permutations for shapley value estimation},
  author={Mitchell, Rory and Cooper, Joshua and Frank, Eibe and Holmes, Geoffrey},
  journal={Journal of Machine Learning Research},
  volume={23},
  number={43},
  pages={1--46},
  year={2022}
}

@article{wang2023note,
  title={A Note on" Towards Efficient Data Valuation Based on the Shapley Value''},
  author={Wang, Jiachen T and Jia, Ruoxi},
  journal={arXiv preprint arXiv:2302.11431},
  year={2023}
}

@article{covert2024stochastic,
  title={Stochastic amortization: A unified approach to accelerate feature and data attribution},
  author={Covert, Ian and Kim, Chanwoo and Lee, Su-In and Zou, James Y and Hashimoto, Tatsunori B},
  journal={Advances in Neural Information Processing Systems},
  volume={37},
  pages={4374--4423},
  year={2024}
}

@article{cook1980characterizations,
  title={Characterizations of an empirical influence function for detecting influential cases in regression},
  author={Cook, R Dennis and Weisberg, Sanford},
  journal={Technometrics},
  volume={22},
  number={4},
  pages={495--508},
  year={1980},
  publisher={Taylor \& Francis}
}

@article{grosse2023studying,
  title={Studying large language model generalization with influence functions},
  author={Grosse, Roger and Bae, Juhan and Anil, Cem and Elhage, Nelson and Tamkin, Alex and Tajdini, Amirhossein and Steiner, Benoit and Li, Dustin and Durmus, Esin and Perez, Ethan and others},
  journal={arXiv preprint arXiv:2308.03296},
  year={2023}
}

@article{xu2021validation,
  title={Validation free and replication robust volume-based data valuation},
  author={Xu, Xinyi and Wu, Zhaoxuan and Foo, Chuan Sheng and Low, Bryan Kian Hsiang},
  journal={Advances in Neural Information Processing Systems},
  volume={34},
  pages={10837--10848},
  year={2021}
}

@inproceedings{wu2022davinz,
  title={Davinz: Data valuation using deep neural networks at initialization},
  author={Wu, Zhaoxuan and Shu, Yao and Low, Bryan Kian Hsiang},
  booktitle={International Conference on Machine Learning},
  pages={24150--24176},
  year={2022},
  organization={PMLR}
}

@article{radford2019gpt2,
  title={Language Models are Unsupervised Multitask Learners},
  author={Radford, Alec and Wu, Jeffrey and Child, Rewon and Luan, David and Amodei, Dario and Sutskever, Ilya},
  journal={OpenAI Technical Report},
  year={2019}
}

@article{sanh2019distilbert,
  title={DistilBERT, a distilled version of BERT: smaller, faster, cheaper and lighter},
  author={Sanh, Victor and Debut, Lysandre and Chaumond, Julien and Wolf, Thomas},
  journal={arXiv preprint arXiv:1910.01108},
  year={2019}
}

@article{merity2017pointer,
  title={Pointer Sentinel Mixture Models},
  author={Merity, Stephen and Xiong, Caiming and Bradbury, James and Socher, Richard},
  journal={arXiv preprint arXiv:1609.07843},
  year={2017}
}

@inproceedings{socher2013recursive,
  title={Recursive Deep Models for Semantic Compositionality Over a Sentiment Treebank},
  author={Socher, Richard and Perelygin, Alex and Wu, Jean and Chuang, Jason and Manning, Christopher D and Ng, Andrew and Potts, Christopher},
  booktitle={EMNLP},
  year={2013}
}
\bibliographystyle{plainnat}

\newpage
\appendix

\vspace{-0.1in}\section{In-Run Data Shapley under Adam}\label{app:adam}

\vspace{-0.1in}
\subsection{Proof of Lemma~\ref{lem:adam_sv}}

\begin{lemma}[Restatement of Lemma~\ref{lem:adam_sv}]
\label{lem:adam_sv_app}
For any iteration $t=0,\ldots,T-1$, under the fixed-state Adam local counterfactual, the first-order approximation of the Adam local utility is
\[
U_{\mathrm{Adam}}^{(t)}(S;z^{(\mathrm{val})})
\approx
-\eta_t
\left\langle
\nabla \ell(w_t,z^{(\mathrm{val})}),
\frac{m_t(S)}{\sqrt{v_t(S)}+\epsilon}
\right\rangle .
\]
\end{lemma}

\begin{proof}
By definition, the Adam local utility is
\[
U_{\mathrm{Adam}}^{(t)}(S;z^{(\mathrm{val})})
=
\ell(\widetilde w_{t+1}(S),z^{(\mathrm{val})})
-
\ell(w_t,z^{(\mathrm{val})}).
\]
Applying a first-order Taylor expansion of the validation loss around $w_t$ gives
\[
\ell(\widetilde w_{t+1}(S),z^{(\mathrm{val})})
-
\ell(w_t,z^{(\mathrm{val})})
=
\left\langle
\nabla \ell(w_t,z^{(\mathrm{val})}),
\widetilde w_{t+1}(S)-w_t
\right\rangle
+
R_t(S),
\]
where $R_t(S)$ contains the higher-order remainder terms. 
Under the fixed-state Adam local counterfactual, the previous moments are held fixed and the current coalition gradient induces
\[
\widetilde w_{t+1}(S)
=
w_t
-
\eta_t
\frac{m_t(S)}{\sqrt{v_t(S)}+\epsilon}.
\]
Therefore,
\[
\widetilde w_{t+1}(S)-w_t
=
-\eta_t
\frac{m_t(S)}{\sqrt{v_t(S)}+\epsilon}.
\]
Substituting this into the Taylor expansion and dropping the higher-order remainder yields
\[
U_{\mathrm{Adam}}^{(t)}(S;z^{(\mathrm{val})})
\approx
-\eta_t
\left\langle
\nabla \ell(w_t,z^{(\mathrm{val})}),
\frac{m_t(S)}{\sqrt{v_t(S)}+\epsilon}
\right\rangle .
\]
This completes the proof.
\end{proof}

\vspace{-0.1in}\subsection{Efficient Computation via Linearized Ghost Dot-Product}
In this subsection, we provide full details about the Ghost Dot Product for our In-Run Data Shapley for Adam.

\vspace{-0.1in}\subsubsection{Proof of \cref{thm:adam_sv_approx}}

\begin{theorem}[Restatement of Theorem~\ref{thm:adam_sv_approx}]
\label{thm:adam_sv_approx_app}
Let $\bar g_t$ be a fixed reference gradient and $J_t
:= \left. \frac{\partial \Psi_t(g)}{\partial g} \right|_{g=\bar g_t}.$
Under the Linearized Ghost Approximation $\Psi_t(g_t(S))
\approx \Psi_t(\bar g_t)+J_t\big(g_t(S)-\bar g_t\big),$
let $\widetilde U_{\mathrm{Adam}}^{(t)}$ denote the resulting linearized Adam local utility. Then, for any $z\in B_t$, the local In-Run Shapley value satisfies
$\phi_z\!\left(\widetilde U_{\mathrm{Adam}}^{(t)}\right)
=
-\eta_t
\left\langle
\nabla \ell(w_t,z^{(\mathrm{val})}),
J_t\nabla \ell(w_t,z)
\right\rangle .$
Therefore, the global Adam-aware In-Run Data Shapley value under the linearized utility
$\widetilde U_{\mathrm{Adam}}=\sum_{t=0}^{T-1}\widetilde U_{\mathrm{Adam}}^{(t)}$
is
\[
\phi_z(\widetilde U_{\mathrm{Adam}})
=
-\sum_{t:z\in B_t}
\eta_t
\left\langle
\nabla \ell(w_t,z^{(\mathrm{val})}),
J_t\nabla \ell(w_t,z)
\right\rangle .
\]
\end{theorem}

\begin{proof}
Fix an iteration $t$ and a validation point $z^{(\mathrm{val})}$. 
For any coalition $S\subseteq B_t$, define the coalition gradient
\[
g_t(S):=\sum_{z\in S}\nabla \ell(w_t,z).
\]
By Lemma~\ref{lem:adam_sv}, the first-order Adam local utility can be written as
\[
U_{\mathrm{Adam}}^{(t)}(S)
\approx
-\eta_t
\nabla \ell(w_t,z^{(\mathrm{val})})^\top
\Psi_t(g_t(S)),
\]
where
\[
\Psi_t(g)
:=
\frac{
\beta_1m_{t-1}+(1-\beta_1)g
}{
\sqrt{\beta_2v_{t-1}+(1-\beta_2)g^{\odot 2}}+\epsilon
}.
\]

Under the Linearized Ghost Approximation, we linearize $\Psi_t$ around a fixed reference gradient $\bar g_t$:
\[
\Psi_t(g)
\approx
\Psi_t(\bar g_t)+J_t(g-\bar g_t),
\qquad
J_t
:=
\left.
\frac{\partial \Psi_t(g)}{\partial g}
\right|_{g=\bar g_t}.
\]
Substituting this linearization into the Adam local utility defines the linearized local utility
\[
\widetilde U_{\mathrm{Adam}}^{(t)}(S)
:=
-\eta_t
\nabla \ell(w_t,z^{(\mathrm{val})})^\top
\left[
\Psi_t(\bar g_t)+J_t\big(g_t(S)-\bar g_t\big)
\right].
\]
Rearranging terms gives
\[
\widetilde U_{\mathrm{Adam}}^{(t)}(S)
=
c_t
-
\eta_t
\nabla \ell(w_t,z^{(\mathrm{val})})^\top
J_t g_t(S),
\]
where
\[
c_t
:=
-\eta_t
\nabla \ell(w_t,z^{(\mathrm{val})})^\top
\Psi_t(\bar g_t)
+
\eta_t
\nabla \ell(w_t,z^{(\mathrm{val})})^\top
J_t\bar g_t
\]
is independent of the coalition $S$.

Now consider any $z\in B_t$ and any coalition $S\subseteq B_t\setminus\{z\}$. 
Since
\[
g_t(S\cup\{z\})=g_t(S)+\nabla\ell(w_t,z),
\]
we have
\[
\begin{aligned}
\widetilde U_{\mathrm{Adam}}^{(t)}(S\cup\{z\})
-
\widetilde U_{\mathrm{Adam}}^{(t)}(S)
&=
-\eta_t
\nabla \ell(w_t,z^{(\mathrm{val})})^\top
J_t
\left(
g_t(S\cup\{z\})-g_t(S)
\right)  \\
&=
-\eta_t
\nabla \ell(w_t,z^{(\mathrm{val})})^\top
J_t
\nabla \ell(w_t,z).
\end{aligned}
\]
The right-hand side is independent of the coalition $S$. 
Therefore, when substituted into the Shapley definition, every marginal contribution term is the same, and the Shapley weights sum to one. Hence,
\[
\phi_z\!\left(\widetilde U_{\mathrm{Adam}}^{(t)}\right)
=
-\eta_t
\nabla \ell(w_t,z^{(\mathrm{val})})^\top
J_t
\nabla \ell(w_t,z).
\]

Finally, by the linearity of the Shapley value,
\[
\phi_z(\widetilde U_{\mathrm{Adam}})
=
\sum_{t=0}^{T-1}
\phi_z\!\left(\widetilde U_{\mathrm{Adam}}^{(t)}\right).
\]
Since $\phi_z(\widetilde U_{\mathrm{Adam}}^{(t)})=0$ whenever $z\notin B_t$, we obtain
\[
\phi_z(\widetilde U_{\mathrm{Adam}})
=
-\sum_{t:z\in B_t}
\eta_t
\left\langle
\nabla \ell(w_t,z^{(\mathrm{val})}),
J_t\nabla \ell(w_t,z)
\right\rangle .
\]
This completes the proof.
\end{proof}

\vspace{-0.1in}\subsubsection{Computation of $J_t$}

We provide the explicit form of the Jacobian used in the Linearized Ghost Approximation. 
Recall the Adam update map with respect to the coalition gradient $g$:
\[
\Psi_t(g)
:=
\frac{
\beta_1 m_{t-1}+(1-\beta_1)g
}{
\sqrt{\beta_2 v_{t-1}+(1-\beta_2)g^{\odot 2}}+\epsilon
},
\]
where all operations are elementwise. 
Since $\Psi_t$ is coordinatewise separable, its Jacobian with respect to $g$ is diagonal.

For the $j$-th coordinate, define
\[
a_j:=\beta_1(m_{t-1})_j,\qquad
b:=1-\beta_1,\qquad
c_j:=\beta_2(v_{t-1})_j,\qquad
d:=1-\beta_2 .
\]
Then
\[
\Psi_{t,j}(g_j)
=
\frac{a_j+b g_j}{\sqrt{c_j+d g_j^2}+\epsilon}.
\]
Differentiating with respect to $g_j$ gives
\[
\Psi_{t,j}'(g_j)
=
\frac{b}{\sqrt{c_j+d g_j^2}+\epsilon}
-
\frac{(a_j+b g_j)d g_j}
{
\sqrt{c_j+d g_j^2}
\left(\sqrt{c_j+d g_j^2}+\epsilon\right)^2
}.
\]

Therefore, for a fixed reference gradient $\bar g_t$, the Jacobian $J_t
:= \left. \frac{\partial \Psi_t(g)}{\partial g} \right|_{g=\bar g_t}$
is diagonal. 
Its $j$-th diagonal entry is
\[
(J_t)_{jj}
=
\frac{1-\beta_1}
{
\sqrt{\beta_2(v_{t-1})_j+(1-\beta_2)\bar g_{t,j}^2}+\epsilon
}
-
\frac{
\left(\beta_1(m_{t-1})_j+(1-\beta_1)\bar g_{t,j}\right)
(1-\beta_2)\bar g_{t,j}
}
{
\sqrt{\beta_2(v_{t-1})_j+(1-\beta_2)\bar g_{t,j}^2}
\left(
\sqrt{\beta_2(v_{t-1})_j+(1-\beta_2)\bar g_{t,j}^2}
+\epsilon
\right)^2
}.
\]

Equivalently, define the reference moments
\[
\bar m_t:=\beta_1m_{t-1}+(1-\beta_1)\bar g_t,
\qquad
\bar v_t:=\beta_2v_{t-1}+(1-\beta_2)\bar g_t^{\odot 2}.
\]
Then
\[
(J_t)_{jj}
=
\frac{1-\beta_1}{\sqrt{(\bar v_t)_j}+\epsilon}
-
\frac{
(\bar m_t)_j(1-\beta_2)\bar g_{t,j}
}{
\sqrt{(\bar v_t)_j}
\left(\sqrt{(\bar v_t)_j}+\epsilon\right)^2
}.
\]
In vector form,
\[
J_t
=
\operatorname{Diag}\!\left(
\frac{1-\beta_1}{\sqrt{\bar v_t}+\epsilon}
-
\frac{
\bar m_t\odot(1-\beta_2)\bar g_t
}{
\sqrt{\bar v_t}\odot(\sqrt{\bar v_t}+\epsilon)^2
}
\right),
\]
where all products, divisions, square roots, and powers are elementwise.

\section{Additional Experimental Setup for Section~\ref{sec:optimizer_dependence}}\label{sec:append_setup_opt}
\noindent\textbf{Experimental Setup.} We trained a GPT-2 Small model (124M parameters) using both SGD and AdamW optimizers. To estimate the TMC values, we sampled multiple random permutations of the training data and tracked the marginal contribution of 10,000 samples to the validation loss throughout the training process.

\vspace{-0.1in}\section{Additional Experiment: Data Pruning Effect}
\label{app:data_pruning}

\begin{figure}[t]
    \centering
    \includegraphics[width=0.72\linewidth]{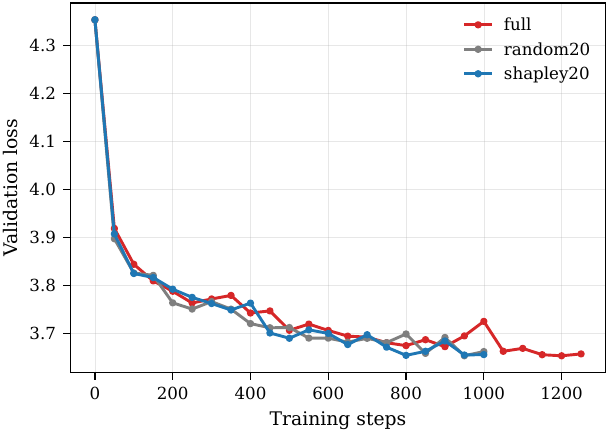}
    \caption{
    \textbf{Data pruning effect on language model training.}
    Validation loss as a function of training steps for three settings:
    training on the full dataset (\texttt{full}), random pruning of 20\%
    training samples (\texttt{random20}), and In-Run Shapley pruning of
    20\% training samples (\texttt{shapley20}). Lower validation loss is
    better.
    }
    \label{fig:pruning_curve}
\end{figure}

\noindent\textbf{Task.}
We evaluate the effect of data pruning on language model training dynamics.
The task is causal language modeling, and performance is measured by
validation loss during training.

\noindent\textbf{Experimental Setup.}
We train \texttt{DistilGPT2} on a fixed training corpus and evaluate
on a fixed validation set.
All runs use the same model architecture, optimizer (AdamW),
learning rate, batch size, number of epochs, and evaluation frequency.
Validation loss is logged periodically throughout training.

Three training settings are compared:
(1) training on the full dataset (\texttt{full}),
(2) training after randomly pruning 20\% of the training samples
(\texttt{random20}), and
(3) training after pruning 20\% of the training samples with the
lowest In-Run Shapley scores (\texttt{shapley20}).

\noindent\textbf{In-Run Shapley Pruning.}
In-Run Shapley scores are computed during a preliminary training pass
using the Adam-aware ghost approximation.
Samples with the lowest scores are removed prior to retraining.
Random pruning removes the same number of samples uniformly at random,
thereby serving as a budget-matched control that isolates the effect of
which samples are removed from the effect of reducing the training set size.

\noindent\textbf{Results.}
Figure~\ref{fig:pruning_curve} shows validation loss as a function of
training steps.
All three runs exhibit a similar overall convergence trend, indicating
that removing 20\% of the training samples does not destabilize training.
Compared with the full-data baseline, the Shapley-pruned run achieves
lower validation loss over much of the middle and later stages of training.
Compared with random pruning, Shapley pruning remains competitive and
often attains lower validation loss under the same pruning budget.
These results suggest that Adam-aware In-Run Shapley identifies samples
whose removal is at least as effective as random removal and can improve
the training trajectory without increasing the pruning ratio.

\vspace{-0.1in}\subsection{SST-2 Pruning Accuracy Statistics}
\label{app:sst2_pruning_details}

\noindent\textbf{Task.}
We further report detailed statistics for the SST-2 pruning experiment
used in the main paper.
The task is binary sentiment classification on SST-2 using \texttt{DistilBERT}.
Performance is measured by validation accuracy.

\noindent\textbf{Experimental Setup.}
We train \texttt{DistilBERT} on a subset of 10,000 SST-2 training examples.
For AdamW pruning, we compare four settings:
training on the full dataset, bottom pruning, random pruning, and top pruning.
Bottom pruning removes samples with the lowest Adam-aware In-Run Shapley
scores, while top pruning removes samples with the highest scores.
Random pruning removes the same fraction of samples uniformly at random.

For bottom and top pruning, results are averaged over three seeds.
For random pruning, we use five random pruning repeats for each seed,
resulting in fifteen runs in total.
The full-data baseline is averaged over three seeds.

\begin{table}[h]
\centering
\caption{
\textbf{Extended AdamW pruning statistics on SST-2.}
We report validation accuracy as mean $\pm$ standard deviation.
The full-data baseline is included for reference.
}
\label{tab:app_sst2_pruning_adam}
\resizebox{0.75\linewidth}{!}{
\begin{tabular}{l|c|c|c}
\toprule
\textbf{Setting} & \textbf{Prune Ratio} & \textbf{Validation Accuracy} & \textbf{\# Runs} \\
\midrule
Full data & 0\% & $0.8719 \pm 0.0058$ & 3 \\
\midrule
Bottom pruning & 10\% & $\mathbf{0.8838 \pm 0.0052}$ & 3 \\
Random pruning & 10\% & $0.8713 \pm 0.0092$ & 15 \\
Top pruning & 10\% & $0.8685 \pm 0.0046$ & 3 \\
\midrule
Bottom pruning & 20\% & $\mathbf{0.8826 \pm 0.0074}$ & 3 \\
Random pruning & 20\% & $0.8752 \pm 0.0075$ & 15 \\
Top pruning & 20\% & $0.8639 \pm 0.0046$ & 3 \\
\midrule
Bottom pruning & 30\% & $\mathbf{0.8876 \pm 0.0061}$ & 3 \\
Random pruning & 30\% & $0.8739 \pm 0.0069$ & 15 \\
Top pruning & 30\% & $0.8532 \pm 0.0140$ & 3 \\
\bottomrule
\end{tabular}
}
\end{table}

\noindent\textbf{Results.}
Table~\ref{tab:app_sst2_pruning_adam} provides the detailed statistics
behind the pruning results reported in the main paper.
The full-data AdamW baseline achieves $0.8719 \pm 0.0058$ validation
accuracy.
Across all pruning ratios, bottom pruning consistently outperforms both
random pruning and top pruning.
At pruning ratios of 10\%, 20\%, and 30\%, bottom pruning achieves
$0.8838$, $0.8826$, and $0.8876$ validation accuracy, respectively,
all exceeding the full-data baseline.
This suggests that the lowest-scored samples identified by Adam-aware
In-Run Shapley are not merely redundant, but can actively reduce
generalization performance when retained.

In contrast, top pruning consistently hurts performance, especially at
higher pruning ratios.
At 30\% pruning, top pruning drops to $0.8532 \pm 0.0140$,
substantially below both bottom pruning and random pruning.
This supports the interpretation that high-score samples are beneficial
for training and should not be removed.
\vspace{-0.1in}\section{Additional Semantic Source Identification Results}
\label{app:semantic_1k}

\noindent\textbf{Small Candidate-Pool Setting.}
In addition to the main semantic source identification experiment with
$10\,\mathrm{k}$ candidate training examples, we also evaluate a smaller
candidate-pool setting with $1\,\mathrm{k}$ candidates.
This experiment serves as a sanity check for the ranking protocol under
a lower retrieval difficulty.
As in the main experiment, lower rank indicates better identification of
the true source example.

\begin{table}[h]
\centering
\caption{
\textbf{Semantic source identification with $1\,\mathrm{k}$ candidate
training examples.}
We report the average rank of the true source sample $z^{\ast}$.
Lower is better. BM25 is a lexical retrieval baseline, IF-Proxy is a
post-training gradient-dot-product influence proxy, and SGD-InRun /
Adam-InRun compare optimizer-aware in-run attribution.
}
\label{tab:semantic_robustness_1k}
\resizebox{0.75\linewidth}{!}{
\begin{tabular}{lcccc}
\toprule
\textbf{Scenario}
& \textbf{BM25}
& \textbf{IF-Proxy}
& \textbf{SGD-InRun}
& \textbf{Adam-InRun} \\
\midrule
Exact               & \textbf{1.00} & \textbf{1.00} & \textbf{1.00} & \textbf{1.00} \\
Partial             & \textbf{1.00} & 4.00  & 2.00  & 1.50 \\
Paraphrase          & \textbf{1.00} & 8.00  & 1.75  & 4.50 \\
Significant rewrite & \textbf{1.00} & 56.25 & 13.75 & 4.75 \\
Similar topic       & 618.00        & 658.75 & 302.00 & \textbf{81.50} \\
\bottomrule
\end{tabular}
}
\end{table}

\noindent\textbf{Discussion.}
Table~\ref{tab:semantic_robustness_1k} shows that the $1\,\mathrm{k}$
candidate-pool setting is substantially easier than the main
$10\,\mathrm{k}$ setting.
For exact, partial, paraphrase, and significant-rewrite variants, BM25
often ranks the true source at or near the top because lexical overlap
remains highly informative in the smaller candidate pool.
This confirms that BM25 should be interpreted as a lexical retrieval
baseline rather than as an optimizer-aware attribution method.

The optimizer-aware comparison is more informative in the semantically
challenging variants.
Under significant rewriting, Adam-InRun improves the average rank over
SGD-InRun from $13.75$ to $4.75$.
Under similar-topic queries, where surface-form matching becomes much
less reliable, Adam-InRun improves the average rank from $302.00$ to
$81.50$.
These results are consistent with the main $10\,\mathrm{k}$ experiment
and provide additional evidence that Adam-aware in-run attribution better
captures optimizer-dependent semantic influence than the SGD-style proxy.
\vspace{-0.1in}\section{Semantic Variant Generation via Prompt Engineering}
\label{sec:semantic_variant_generation}

To probe robustness under controlled semantic perturbations,
we generate structured validation variants conditioned on a source training sample $z^{\ast}$.
Each $z^{\ast}$ is associated with four categories of validation texts:
(i) \texttt{partial\_same}, (ii) \texttt{paraphrase},
(iii) \texttt{significant\_paraphrase}, and (iv) \texttt{similar\_topic}.
These categories are designed to progressively decouple surface form, entity usage,
and semantic content, enabling fine-grained evaluation of attribution stability.

\noindent\textbf{Structured Prompt-Based Generation.}
We employ a large language model (DeepSeek-V3) as a controllable text generator.
To ensure structural consistency and downstream reproducibility,
all generations are constrained to a \texttt{json\_object} response format.
The prompt consists of a fixed system instruction and a parameterized user template,
where the source sample $z^{\ast}$ is explicitly injected as a variable.

\noindent\textbf{System prompt.}
\begin{quote}\small
You are generating controlled text variants for a machine learning experiment.
Return ONLY valid JSON. Do NOT include markdown, explanations, or extra text.
\end{quote}

\noindent\textbf{User prompt template.}
\begin{quote}\small
Given the following text $z^{\ast}$, generate text variants in four categories:
\texttt{partial\_same}, \texttt{paraphrase},
\texttt{significant\_paraphrase}, and \texttt{similar\_topic}.
Each output must be a single English sentence and satisfy category-specific constraints.
The model must return a JSON object with exactly these four keys.
\end{quote}

The full \emph{verbatim} user prompt, including category definitions,
forbidden-token injection, and formatting constraints,
is provided in Appendix~\ref{app:semantic_source_id_prompt} for exact reproducibility.

\vspace{-0.1in}\subsection{Deterministic Filtering and Quality Control}
\label{sec:semantic_filtering}

Raw generations from the language model are further processed through a deterministic
multi-stage filtering pipeline to remove trivial, degenerate, or overly lexical variants.
This ensures that robustness evaluation reflects genuine semantic perturbations rather
than surface-level artifacts.

\noindent\textbf{Lexical Constraint Enforcement.}
We extract a set of forbidden tokens
$\mathcal{T}_{\mathrm{forbid}}$ from $z^{\ast}$,
including named entities, alphanumeric identifiers, and long high-entropy content words.
Candidates in the \texttt{significant\_paraphrase} and \texttt{similar\_topic} categories
must satisfy
\[
\mathcal{W}(C) \cap \mathcal{T}_{\mathrm{forbid}} = \emptyset,
\]
where $\mathcal{W}(C)$ denotes the word set of candidate $C$.
This constraint prevents entity memorization and enforces semantic abstraction.

\noindent\textbf{Jaccard Similarity Thresholding.}
We measure lexical overlap using word-level Jaccard similarity
$J(C, z^{\ast})$.
Each category enforces a distinct overlap regime:
\begin{itemize}
    \item \texttt{Paraphrase:}
    $\tau_{p,\min} \le J(C, z^{\ast}) \le \tau_{p,\max}$,
    ensuring semantic equivalence with moderate surface variation.
    \item \texttt{Significant paraphrase:}
    $J(C, z^{\ast}) \le \tau_{s,\max}$,
    enforcing substantial rewriting while preserving meaning.
    \item \texttt{Similar topic:}
    an even stricter upper bound on $J(C, z^{\ast})$,
    encouraging topical relatedness without semantic identity.
\end{itemize}

\noindent\textbf{Final Selection.}
After filtering, we deduplicate candidates and select a fixed number
($K_{\mathrm{each}}$) per category.
All randomness is controlled by a fixed global seed, and generated scenarios
are cached to disk to guarantee deterministic reuse across runs.

\vspace{-0.1in}\subsection{Verbatim Prompt for Semantic Variant Generation}
\label{app:semantic_source_id_prompt}

\noindent\textbf{System prompt:}
\begin{quote}\small\ttfamily
You are generating controlled text variants for a machine learning experiment.
Return ONLY valid JSON. Do NOT include markdown, explanations, or extra text.
\end{quote}

\noindent\textbf{User prompt:}
\begin{quote}
\begin{verbatim}
Given the following text z_star, generate text variants in four categories.

Return a JSON object with EXACTLY the following keys:
- "partial_same"
- "paraphrase"
- "significant_paraphrase"
- "similar_topic"

Each value must be a list of strings.

Constraints:
- Each string must be ONE English sentence.
- Do not copy z_star verbatim except for "partial_same".
- "partial_same": truncated or lightly edited version of z_star (1 item).
- "paraphrase": same meaning, mild rewriting (at least K candidates).
- "significant_paraphrase": same meaning, substantial rewriting,
avoid reusing entities.
- "similar_topic": same general topic, different facts/details,
minimal lexical overlap.

Forbidden tokens:
[FORBIDDEN_TOKEN_LIST]

z_star:
[SOURCE TEXT]
\end{verbatim}
\end{quote}

\section{Limitations and Broader Impacts}

\paragraph{Limitations.}
Our method targets the step-wise In-Run Data Shapley utility rather than full unrolled final-checkpoint influence. 
It relies on first-order Taylor approximations and a linearization of the Adam update map around a reference gradient. 
The empirical evaluation focuses on language modeling and text classification tasks with AdamW/SGD; extending the study to larger pretraining runs and additional optimizers is an important direction.

\paragraph{Broader impacts.}
Optimizer-aware attribution can support data debugging, source tracing, and data curation, which may help reduce harmful or low-quality training data. 
However, attribution scores can be noisy or misinterpreted, and using them for automated data removal may unintentionally discard rare or underrepresented examples. 
In high-stakes settings, such scores should be used as diagnostic evidence rather than as the sole basis for decisions.

\end{document}